\documentclass{article}

\usepackage{microtype}
\usepackage{graphicx}
\usepackage{subcaption}
\usepackage{booktabs} % for professional tables
\usepackage{mathrsfs}
\usepackage{subcaption}
\usepackage{hyperref}

\usepackage[preprint]{icml2026}

\usepackage{amsmath}
\usepackage{amssymb}
\usepackage{mathtools}
\usepackage{amsthm}
\usepackage{algorithm}
\usepackage{algorithmic}

\usepackage[capitalize,noabbrev]{cleveref}

\theoremstyle{plain}
\newtheorem{theorem}{Theorem}[section]
\newtheorem{proposition}[theorem]{Proposition}

\theoremstyle{definition}

\theoremstyle{remark}
\newtheorem{remark}[theorem]{Remark}

\usepackage[textsize=tiny]{todonotes}

\icmltitlerunning{Submission and Formatting Instructions for ICML 2026}

\begin{document}

\twocolumn[
  \icmltitle{Learning Graphs through Continuous Information Entropy Fields}



  \icmlsetsymbol{equal}{*}

  \begin{icmlauthorlist}
    \icmlauthor{Hui Cong}{yyy}%conghui@chd.edu.cn
    \icmlauthor{Bo Sun}{yyy}%sunbo\_bosun@chd.edu.cn
    \icmlauthor{Yaxian Wang}{yyy}%wyx1566@chd.edu.cn
    \icmlauthor{Ziheng Jiao}{sch}
    \icmlauthor{Yisheng An}{yyy}
  \end{icmlauthorlist}

  \icmlaffiliation{yyy}{School of Information Engineering, Chang'an University, Shaanxi, Xi'an, China}
  %\icmlaffiliation{comp}{Company Name, Location, Country}
  %\icmlaffiliation{sch}{School of ZZZ, Institute of WWW, Location, Country}

  \icmlcorrespondingauthor{Yisheng An}{aysm@chd.edu.cn}
  %\icmlcorrespondingauthor{Firstname2 Lastname2}{first2.last2@www.uk}

  \icmlkeywords{Graph Neural Networks, Information Entropy, Field Theory, Co-evolution}

  \vskip 0.3in
]



\printAffiliationsAndNotice{}  % no special notice (required even if empty)

\begin{abstract}
Graph theory is inherently descriptive, capturing what relationships exist but not why they arise, because it treats edges as primitive constructs. This paper proposes a new explanatory framework for graph learning, where relationships emerge from latent continuous information entropy fields, and a graph becomes a discrete instantiation of an underlying field. To formalize this field, we introduce the Field-informed Graph Network (FGN). It learns a scalar field from node features and leverages it to modulate message passing. The information-theoretic objective balances structural fidelity with field smoothness, forming a self-reinforcing loop. In this loop, the field modulates information diffusion through field-modulated weighting, and the updated node representations iteratively refine the field. As a result, FGN learns by simulating its own co-evolution. Extensive experiments on node classification and graph classification benchmarks demonstrate superior performance, robustness to perturbations, and structurally coherent field representations.
\end{abstract}

\section{Introduction}
Graphs serve as a foundational language for modeling relational systems, from social interactions to biological networks. The central aim of Graph Representation Learning (GRL) is to map these high-dimensional, non-Euclidean structures into low-dimensional, informative embeddings \cite{GCN}. While Graph Neural Networks (GNNs) have achieved significant success \cite{MDGNN,BAGCN}, the prevailing approach remains largely descriptive. It excels at answering ``what is the structure" but offers limited insight into the probabilistic mechanisms of ``why it forms", leaving the explanation of connection origins an open question \cite{WanyuLin}.

\begin{figure}[!ht]
    \centering
    \begin{subfigure}{0.49\linewidth}
        \includegraphics[width=\linewidth]{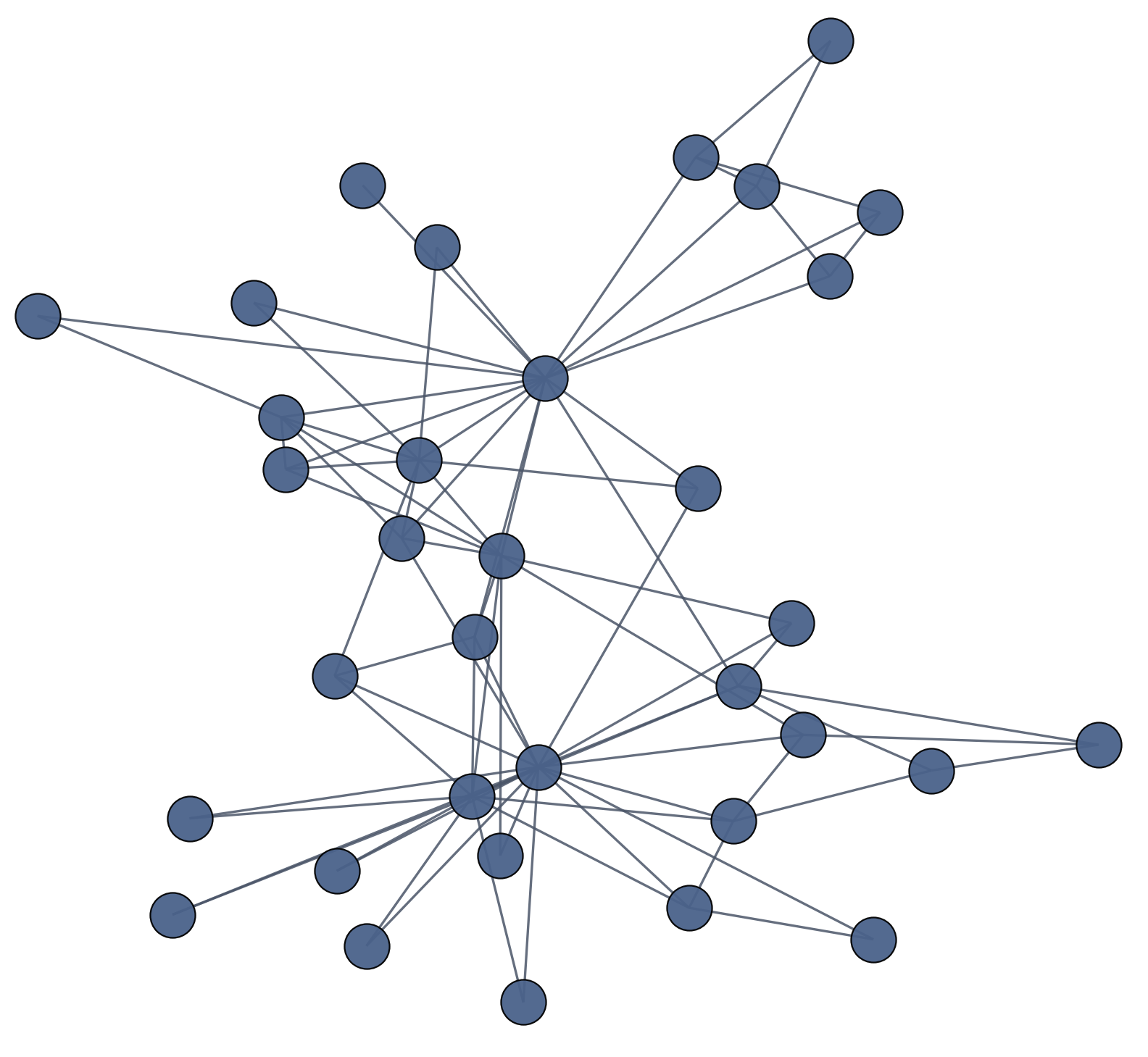}
        \caption{Karate Club network.}
    \end{subfigure}
    \begin{subfigure}{0.49\linewidth}
        \includegraphics[width=\linewidth]{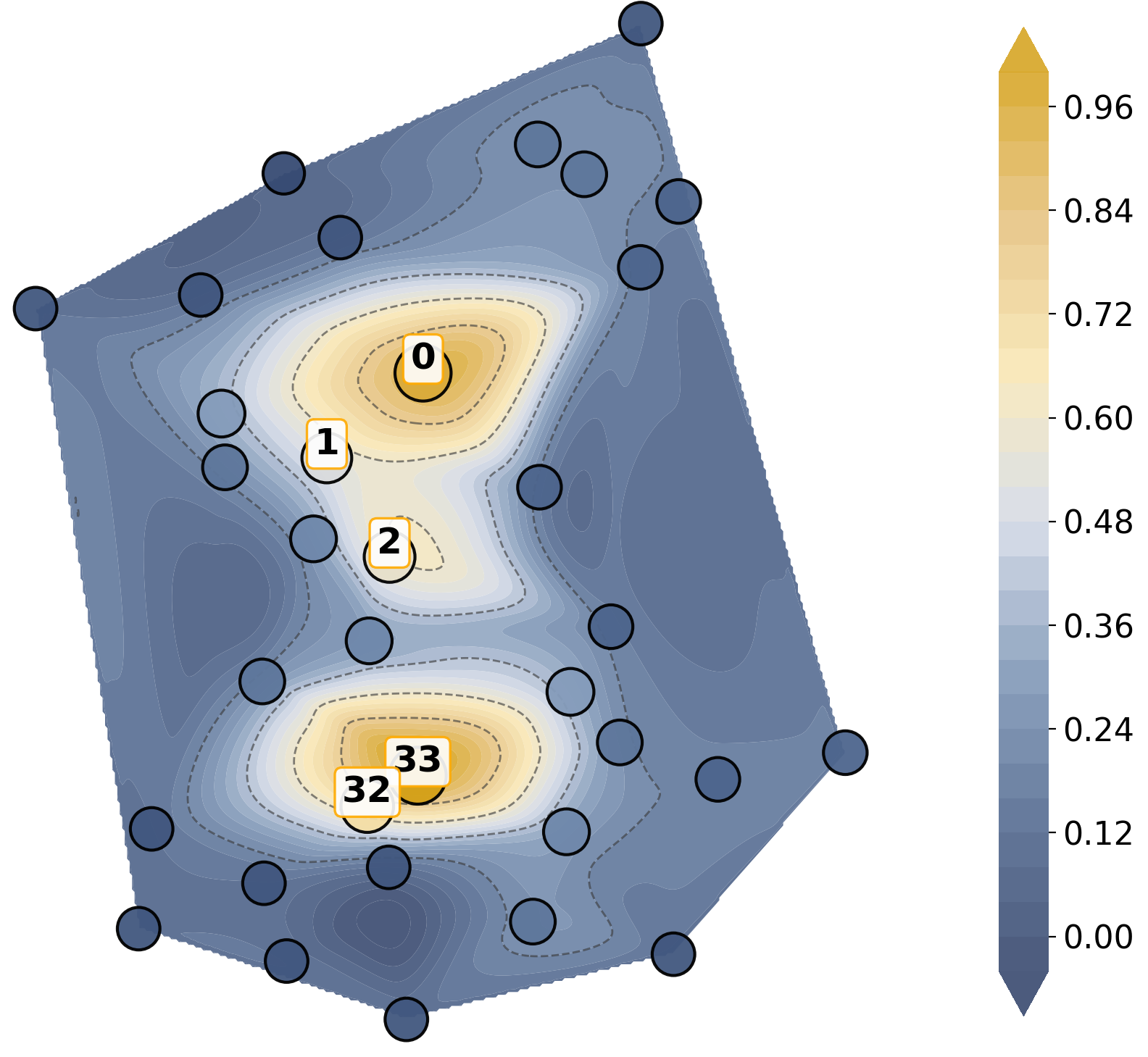}
        \caption{Network with field.}
    \end{subfigure}
    \caption{Discrete projection diagram of the information field (b) inferred from the Karate Club network (a). Node colors denote the relational field of each node, with warmer colors indicating the higher and cooler colors the lower
    . The field exhibits smooth variation within communities and sharp transitions at boundaries, demonstrating how graph structure can emerge from an underlying continuous field.}
    \label{fig:problem}
\end{figure}

The explanatory gap stems from treating relationships as primitive elements. In reality, relational ties emerge from latent continuous fields, such as social influence potentials, biomolecular gradients, or information landscapes \cite{AIFlow, QingyunSun}. Consequently, a graph can be understood as a discrete sampling of an underlying field, an observable imprint left by continuous dynamical processes \cite{SENGraph}.

Despite remarkable advances in graph representation learning, most existing methods remain constrained by discrete analysis \cite{MDGNN,FilippoMariaBianchi}. Descriptive methods \cite{NimrahMustafa, Polarized, MichaelScholkemper} treat edges as given, capturing what connections exist but not why. Probabilistic models, including latent variable models and graph variational autoencoders, often assume edges arise from continuous latent spaces. However, they typically treat the latent variables as static summaries and do not use them to actively modulate message passing during task learning. This perspective is illustrated in Fig.~\ref{fig:problem}, which visualizes an entropic field inferred from the Karate Club network. The field exhibits smoothness within communities and sharp boundaries between them, illustrating how a continuous latent field can relate to discrete graph structure \cite{AliakseiSandryhaila}.

In real-world relational systems, information exchange evolves over time, which implies that the underlying continuous field is not static either. As a result, the correspondence between the graph and the field must also possess dynamic characteristics. A key challenge then is to model the co-evolution between the discrete graph and its continuously adapting field. The field influences the likelihood of connections, and the observed structure in turn reshapes the field \cite{MatejZecevic}. If graphs are emergent from continuous fields, a model should capture this bidirectional relationship.

However, prevailing GNN architectures are not designed for this co-evolutionary modeling. Their reliance on local message passing and unidirectional field coupling creates an informational asymmetry. Graph topology dominates the learning process, while any latent field is relegated to a passive prior \cite{NimrahMustafa, NAAM}.

To bridge this gap, we propose a new framework for graph learning. Specifically, we introduce the Field-informed Graph Network (FGN). Its core is an information-theoretic objective that balances fidelity to the observed graph topology with smoothness of the latent field, leading the system toward a stable equilibrium. In FGN, the latent field is constructed as an explanatory substrate rather than merely serving as attention coefficients. This objective induces a self-reinforcing dynamical process, consisting of entropy-aware message passing modulated by field coherence and an adaptive field evolution step where graph signals refine the field. The closed-loop architecture actively co-evolves the field as an integral component. Our main contributions are as follows:

\begin{itemize}
    \item To bridge the explanatory gap in graph learning, we introduce a field-theoretic framework. This framework views graphs not as collections of primitive edges, but as discrete samples of an underlying continuous field, thereby establishing the association between the latent information entropy field and the discrete relations.
    \item To operationalize the co-evolution of graphs and fields, we propose FGN, centered on an information-theoretic objective that balances structural fidelity with field regularity. Its realization integrates entropy-aware message passing and adaptive field evolution into a unified, closed-loop dynamical system.
    \item Extensive experiments on node classification and graph classification benchmarks show that FGN achieves strong performance, improved robustness to perturbations, and learns structurally meaningful field representations.
\end{itemize}

\section{Related Works}
\subsection{Descriptive and Probabilistic Models}
Research in graph representation learning has largely followed two directions: descriptive modeling and probabilistic modeling.
The descriptive direction, dominated by GNNs, treats graph topology as a fixed structure for predictive tasks. By propagating information along edges through message passing \cite{GCN,GAT}, these models learn representations that correlate node features with target variables \cite{XuelongLi}. While successful for node classification and link prediction, this approach describes what connections exist but does not explain why they formed \cite{DimitriosKelesis,MichaelScholkemper}. The graph structure remains an input, not a phenomenon to be explained.
The probabilistic direction models the distribution from which graphs are sampled. Approaches include autoregressive edge models, latent variable models such as graph variational autoencoders, and diffusion-based processes \cite{DiGress}. Many of these methods assume edges arise from continuous latent spaces, which aligns with our perspective. However, they typically treat the latent variables as static summaries inferred via maximum likelihood. They do not use these latent variables to actively modulate message passing during task learning. This distinction motivates our approach, where the learned field directly guides information flow and co-evolves with the graph. Classic proximity-based embedding methods like node2vec \cite{node2vec} and DeepWalk \cite{DeepWalk} also learn a continuous latent space, but they share the same limitation of static representations without task-aware modulation.

\subsection{Field-Based and Physics-Inspired Methods}
Beyond probabilistic models, another line of research explicitly incorporates field or physics concepts. The idea that discrete relations stem from continuous dynamics has motivated research incorporating field or physics concepts \cite{YihengXie}.
One line uses continuous representations to describe graph properties \cite{Fuchsgruber}. Examples include hyperbolic embeddings and neural differential equations over nodes \cite{ShenzhiYang,HaoyuLi}. In these methods, continuity is a property of the representation space, while the graph topology remains static.
Another line introduces physical or probabilistic fields as inductive biases. Classical work models networks with attractive and repulsive forces. Modern approaches include energy-based models that define connection likelihoods via latent kernels, and formulations that cast message passing as a diffusion process on the graph \cite{CGNN, XiaokangZhou}. A common limitation of these methods is that the field is often prescribed rather than learned from task objectives. It serves as a fixed prior, for example a Gaussian kernel, or a simplified mechanistic analogy such as spring forces. The field remains subordinate to the graph, a passive tool rather than a latent variable that co-evolves with the structure.

In contrast, our framework learns an explicit, globally smooth entropy field that directly modulates message passing and is refined by the graph representations under an information-theoretic objective. This is fundamentally different from GAT, whose attention weights are purely local and derived from node features without any underlying field structure, and from diffusion-based models such as GRAND, which rely on feature similarities and do not maintain a learnable, globally coherent field. By jointly evolving the continuous explanatory field and the discrete graph structure in a closed feedback loop, our approach goes beyond existing field-inspired or physics-inspired models, where the field typically remains a static prior or a passive by-product.

\section{Methodology}
\subsection{A Continuous Field Formalism for Graph Modeling}

\paragraph{The Field Formalism.}\label{sec:generative_formalism}
Underlying our framework is a core hypothesis. A graph is not a fundamental construct, but an observable imprint of a latent continuous entropy field.
Let $\mathcal{G} = (\mathcal{V}, \mathcal{E})$ be a graph where each node $i \in \mathcal{V}$ has an associated feature vector $x_i \in \mathcal{X} \subseteq \mathbb{R}^F$. We view $\mathcal{X}$ as a feature space that may possess intrinsic manifold structure.
Let $\mathcal{M}$ be a latent manifold representing the space of relational potentials, and let $\phi^*: \mathcal{M} \rightarrow \mathbb{R}^d$ be an unknown continuous entropy field defined over this manifold. Each node $i$ is associated with a latent position $\mathbf{z}_i \in \mathcal{M}$, though these positions are generally unobserved.
We summarize the other notations in Appendix \ref{app:notation}.

The fundamental assumption is that edge formation probabilities depend solely on distances in this field space:
\begin{equation}
P(A_{ij}=1 \mid \phi^*) = \sigma\left(-\frac{d_{\mathcal{M}}(\phi^*(\mathbf{z}_i), \phi^*(\mathbf{z}_j))^2}{\tau}\right),
\label{eq:generative_model}
\end{equation}
where $d_{\mathcal{M}}$ is a distance metric on the manifold $\mathcal{M}$, $\sigma(\cdot)$ is the logistic sigmoid function $\sigma(z) = 1/(1+e^{-z})$, and $\tau > 0$ is a temperature parameter controlling the sharpness of the probability distribution. This formulation captures the intuitive notion that nodes with coherent field values are more likely to be connected.

The observed graph $\mathcal{G}$ is then a discrete sample from this Bernoulli distribution over all node pairs:
\begin{equation}
P(\mathcal{G} \mid \phi^*) = \prod_{(i,j)\in \mathcal{E}} P_{ij} \prod_{(i,j)\notin \mathcal{E}} (1-P_{ij}),
\label{eq:graph_likelihood}
\end{equation}
with $P_{ij} \equiv P(\mathbf{A}_{ij}=1 \mid \phi^*)$. In this model, $\phi^*$ is the explanatory substrate. Edges emerge from the continuous field rather than being primitive.

This formulation establishes $\phi^*$ as the explanatory substrate of the graph topology. Edges are not primitives but emerge from the continuous field of relational potential. The temperature $\tau$ has a natural interpretation: as $\tau \to 0$, the sigmoid approaches a step function, making connections deterministic; as $\tau \to \infty$, edge formation becomes random and independent of the field.

\begin{theorem}[Approximation Consistency]
\label{thm:approximation}
Let $\phi^*: \mathcal{M} \rightarrow \mathbb{R}^d$ be $L$-Lipschitz continuous.
Suppose there exists a mapping $\iota: \mathcal{X} \rightarrow \mathcal{M}$ that relates the feature space to the latent manifold with bounded distortion, i.e., there is a constant $\kappa \ge 1$ such that
\begin{equation}
\frac{1}{\kappa} \|x - x'\| \le d_{\mathcal{M}}(\iota(x), \iota(x')) \le \kappa \|x - x'\|, \quad \forall x,x'\in\mathcal{X}.
\end{equation}
Let $\phi_\theta: \mathcal{X} \rightarrow \mathbb{R}^d$ be a learned field with uniform approximation error
$\epsilon = \sup_{x \in \mathcal{X}} \|\phi_\theta(x) - \phi^*(\iota(x))\|$.
Denote by $D = \max_{x,x'\in\mathcal{X}} \|x - x'\|$ the diameter of the feature space.
If $\epsilon \le L\kappa D$, then the total variation distance between the graph distributions induced by $\phi_\theta$ and by the true composition $\phi^* \circ \iota$ satisfies
\begin{equation}
D_{\mathrm{TV}}\big(P(\mathcal{G} \mid \phi_\theta),\, P(\mathcal{G} \mid \phi^* \circ \iota)\big)
\le \frac{3 L \kappa D}{2\tau}\,\epsilon.
\label{eq:tv_bound}
\end{equation}
In particular, the bound can be written as $\frac{C L \kappa}{\tau}\epsilon$ where $C$ depends polynomially on the feature space diameter $D$ and the number of nodes $N$.
\end{theorem}

\begin{remark}
The bounded distortion condition is significantly weaker than an isometric embedding; it only requires that distances in feature space and on the latent manifold are comparable up to a constant factor $\kappa$.
In practice, the mapping $\iota$ is not constructed explicitly. Instead, the neural network $\phi_\theta$ learns the composition $\phi^* \circ \iota$ end-to-end, implicitly absorbing the distortion $\kappa$ into the effective smoothness of the learned field.
Even when the condition holds only approximately, the same functional form of the bound remains valid with $\kappa$ replaced by an effective distortion constant, and the total variation distance degrades gracefully as the deviation from bi-Lipschitz increases.
\end{remark}

%This ensures that with sufficient model capacity, our parameterized field $\phi_\theta$ can approximate the true field $\phi^*$ arbitrarily well, thereby justifying the variational learning approach. The bound tightens with smoother fields and sharper sigmoid responses.
%Intuitively, the entropy field defines a continuous geometric landscape over the latent manifold. Nodes residing in regions of similar field values are more likely to be connected, while large field differences suppress edge formation. The temperature parameter $\tau$ controls the sharpness of this geometric preference: a smaller $\tau$ makes the connection probability more sensitive to field variations, whereas a larger $\tau$ smooths out the landscape. This geometric view, where graph topology emerges from an underlying continuous entropy geometry, forms the core intuition of our approach.

\paragraph{Variational Inference of Latent Fields.}\label{sec:variational_inference}
While the general formalism involves a $d$-dimensional field on a latent manifold, we observe that for many relational inference tasks, the essential explanatory factor can be captured by a scalar field representing the relational potential or energy of a node. FGN naturally accommodates this by setting $d=1$, which simplifies implementation while preserving the core intuition. This reduction not only simplifies implementation and improves interpretability, but also aligns with the intuition that a single energy or potential value can succinctly capture the propensity of a node to form relationships. This scalar choice is sufficient for many real-world graphs, particularly those exhibiting homophily. Higher-dimensional fields may benefit more complex scenarios such as heterophilic or multi-relational graphs, which we leave for future work. The subsequent derivations hold for general $d$, but FGN focuses on the scalar case for clarity and efficiency.

The generative model in Eq.~\eqref{eq:generative_model} presents a fundamental inference challenge: the true entropy field $\phi^*: \mathcal{M} \rightarrow \mathbb{R}^d$ and the latent manifold positions $\{\mathbf{z}_i\}$ are unobservable. We only have access to the discrete graph $\mathcal{G}$ and node features $X = \{x_i\}_{i=1}^n \subset \mathcal{X}$, where each $x_i \in \mathbb{R}^F$ is a feature vector.

As established in Theorem~\ref{thm:approximation}, if there exists a mapping $\iota: \mathcal{X} \rightarrow \mathcal{M}$ associating features to manifold positions, we can learn an approximation to the composition $\phi^* \circ \iota$. The mapping $\iota$ postulated in the theorem is not known a priori; rather, we learn a parameterized field $\phi_\theta: \mathcal{X} \rightarrow \mathbb{R}^d$ that directly approximates $\phi^* \circ \iota$, with the neural network architecture implicitly encoding both $\iota$ and $\phi^*$.

Following the approximation framework, we implement $\phi_\theta$ as a neural network (for simplicity, we denote it as $g_\theta$):
\begin{equation}
\phi_\theta(x_i) = g_\theta(x_i), \quad \forall i \in \mathcal{V},
\label{eq:parameterized_field}
\end{equation}
where $\phi_\theta(x_i)$ directly approximates $\phi^*(\iota(x_i)) \in \mathbb{R}^d$. In the scalar case ($d=1$), $g_\theta$ reduces to the field network $\text{MLP}_{\theta_E}$ defined later in Eq.~\eqref{eq:energy_field}. The network $g_\theta$ implicitly learns both the manifold mapping $\iota$ through its early layers and the field structure $\phi^*$.

The theoretical generative model in Eq.~\eqref{eq:generative_model} defines probabilities via manifold geometry.
For our practical implementation with learned field $\phi_\theta: \mathcal{X} \rightarrow \mathbb{R}^d$, we adopt the Euclidean surrogate:
\begin{equation}
{P_{ij}(\phi_\theta) = \sigma\!\left(-\frac{\|\phi_\theta(x_i) - \phi_\theta(x_j)\|^2}{\tau}\right)},
\label{eq:approx_prob}
\end{equation}
which approximates $P_{ij}(\phi^* \circ \iota) = \sigma(-d_{\mathcal{M}}(\phi^*(\iota(x_i)), \phi^*(\iota(x_j)))^2/\tau)$.
This Euclidean approximation is justified when $\phi^*(\mathcal{M})$ is approximately isometric to a subset of $\mathbb{R}^d$, or when $\iota$ maps to charts where $d_{\mathcal{M}}$ is well-approximated by Euclidean distance.

Formally, we aim to infer the latent composition $\phi^* \circ \iota$. The marginal likelihood of the observed graph can be lower bounded via the Evidence Lower Bound (ELBO):
\begin{align}
\log P(\mathcal{G} \mid X) &= \log \int P(\mathcal{G} \mid \phi \circ \iota, X) P(\phi) P(\iota) d\phi d\iota \nonumber \\
&\geq \mathbb{E}_{(\phi,\iota) \sim Q}[\log P(\mathcal{G} \mid \phi \circ \iota, X)] \nonumber \\
&\quad - \text{KL}(Q(\phi,\iota) \| P(\phi,\iota)),
\label{eq:elbo}
\end{align}
where $Q(\phi,\iota)$ is a variational distribution over fields and mappings. In practice, we adopt a point estimate approximation: $Q(\phi,\iota) = \delta(\phi - \phi_\theta) \delta(\iota - \iota_\theta)$, where $\iota_\theta$ is implicitly defined by $\phi_\theta$.

The regularizer $\mathcal{R}(\phi_\theta)$ in the following objective can be viewed as imposing a specific prior $P(\phi, \iota)$ that favors field smoothness and coherence over the graph structure. This provides a direct link between the variational framework and our practical loss.

\paragraph{From ELBO to a Practical Objective.}
To obtain a tractable objective, we follow common variational practice: the ELBO is approximated with point estimates, and the KL divergence is replaced by an explicit regularizer $\mathcal{R}(\phi_\theta)$ that promotes field coherence, discriminability, and mutual information with the representations.
This yields the regularized loss:
\begin{equation}\label{eq:practical_objective}
\mathcal{L}(\theta) = -\frac{1}{|\mathcal{E}|}\sum_{(i,j)\in\mathcal{E}} \log P_{ij}(\phi_\theta) + \lambda \mathcal{R}(\phi_\theta),
\end{equation}
where the exact form of $\mathcal{R}(\phi_\theta)$ is instantiated in Section~\ref{sec:fgn} and a detailed derivation is provided in Appendix~\ref{app:elbo_derivation}.
\paragraph{Dynamical Co-Evolution Implementation.}\label{sec:dynamical_implementation}
The variational framework naturally induces a dynamical system that operationalizes the co-evolution between graph representations and the learned field $\phi_\theta$. This section describes the iterative algorithm that alternates between field-guided message passing and structure-driven field refinement, realizing the co-evolutionary process motivated above.

At iteration $t$, given the current field estimate $\phi_\theta^{(t)}: \mathcal{X} \rightarrow \mathbb{R}^d$, information propagation through the graph is modulated by field coherence. For a node $i$ with feature $x_i \in \mathcal{X}$, the weight from neighbor $j \in \mathcal{N}(i)$ is computed as:
\begin{equation}
w_{ij}^{(t)} = \exp\left(-\xi \|\phi_\theta^{(t)}(x_i) - \phi_\theta^{(t)}(x_j)\|^2\right),
\label{eq:attention_weights}
\end{equation}
where $\xi > 0$ is a learnable inverse temperature parameter. The node representation update follows:
\begin{equation}
\mathbf{h}_i^{(t+1)} = f_{}\!\left(\mathbf{h}_i^{(t)},\; \sum_{j\in\mathcal{N}(i)} w_{ij}^{(t)} \cdot \psi_{}(\mathbf{h}_j^{(t)})\right),
\label{eq:message_passing}
\end{equation}
with learnable functions $f$ and $\psi$. This implements an entropy-aware message passing where information flow is attenuated between nodes with dissimilar field values, reinforcing the principle that edges correspond to field similarity.

In our FGN implementation, we adopt the standard message-passing update from GraphSAGE~\cite{GraphSAGE}:
$\psi$ is a linear layer, and $f$ performs concatenation followed by a ReLU activation.
Specifically, $\mathbf{h}_i^{(t+1)} = \text{ReLU}\big(\mathbf{W} \cdot [\mathbf{h}_i^{(t)} \,\|\, \sum_{j\in\mathcal{N}(i)} w_{ij} \mathbf{W}_{\text{msg}} \mathbf{h}_j^{(t)}]\big)$,
where $\mathbf{W}$ and $\mathbf{W}_{\text{msg}}$ are learnable matrices.

After obtaining updated node representations $\mathbf{H}^{(t+1)} = \{\mathbf{h}_i^{(t+1)}\}_{i=1}^n$, the field parameters are refined to better explain the observed graph structure. Using the field regularization loss $\mathcal{L}_{\text{field}}$ defined in Eq.~\eqref{eq:practical_objective} and instantiated in Section~\ref{sec:fgn}, we perform a gradient step:
\begin{equation}
\theta^{(t+1)} = \theta^{(t)} - \eta \nabla_\theta \mathcal{L}_{\text{field}}(\phi_\theta^{(t)}, \mathcal{G}),
\label{eq:field_update}
\end{equation}
where $\eta$ is the learning rate. This update encourages $\phi_\theta$ to become more coherent with the current graph representation, specifically promoting alignment between connected nodes as measured by cosine similarity.

For brevity, denote the scalar field value at node $i$ as $\mathbf{E}_i = \phi_\theta(x_i) \in \mathbb{R}$.

\begin{proposition}[Noise suppression via approximate field coherence]
\label{prop:noise_suppression}
Under a relaxed community separation condition where at least a fraction $1-\eta$ of intra-community edges satisfy $|\mathbf{E}_u - \mathbf{E}_v| \le \delta$ and a fraction $1-\eta$ of inter-community edges satisfy $|\mathbf{E}_x - \mathbf{E}_y| \ge \Delta$ (with $\delta<\Delta$, $\eta\in[0,1)$), the field-modulated weights in Eq.~\eqref{eq:attention_weights} exponentially favor intra-community propagation by a factor $e^{\xi(\Delta^2-\delta^2)} - O(\eta)$.
\end{proposition}

\begin{theorem}[Propagation Advantage]
\label{thm:propagation_advantage}
Under the conditions of Proposition~\ref{prop:noise_suppression}, FGN achieves a strictly tighter misclassification upper bound than standard GCN: $\mathbb{P}_{\mathrm{FGN}}(\mathrm{error}) \le \rho(\eta)\,\mathbb{P}_{\mathrm{GCN}}(\mathrm{error})$ with $\rho(\eta)<1$ for small $\eta$, and $\rho(0)=e^{-2\xi(\Delta^2-\delta^2)}<1$. Full statements and proofs are in Appendix~\ref{app:propagation_proof}.
\end{theorem}

\paragraph{Convergence intuition.}
The alternating updates constitute a block coordinate descent on $\mathcal{J}(\mathbf{H},\theta)=\mathcal{L}_{\mathrm{task}}(\mathbf{H})+\lambda\mathcal{L}_{\mathrm{field}}(\phi_\theta,\mathbf{H})$, converging to a stationary point under mild contractivity and smoothness conditions (see Appendix~\ref{app:analysis:coevolution}).

FGN's field-modulated weights differ from standard attention (e.g., GAT~\cite{GAT}) as a globally shared quantity with an independent training objective (see Appendices~\ref{app:attention_relation} and ablation in Sec.~\ref{sec:ablation}).

\subsection{FGN: A Practical Instantiation}\label{sec:fgn}
Building upon the theoretical foundations above, we present FGN, an architecture designed to efficiently approximate the co-evolutionary equilibrium between the graph and its latent field, as formalized by the variational objective.

FGN instantiates the theoretical framework with two key practical choices: (1) modeling a scalar energy field, and (2) encapsulating the iterative co-evolution dynamics into a single forward pass that directly targets the equilibrium state. The core architecture computes node representations $\mathbf{H}$ as:
\begin{equation}
\mathbf{H} = \text{MPNN}\!\left(X,\, A \odot W\right),
\label{eq:fgn_arch}
\end{equation}
where $\text{MPNN}(\cdot)$ denotes a message-passing neural network implementing the field-modulated propagation mechanism from Eq.~\eqref{eq:message_passing}, $X \in \mathbb{R}^{N \times F}$ is the feature matrix, $A \in \{0,1\}^{N \times N}$ the adjacency matrix, $\mathbf{E} \in \mathbb{R}^N$ the learned scalar energy field, and $W = \{w_{ij}\} \in \mathbb{R}^{N \times N}$ is the field-modulated matrix as defined in Eq.~\eqref{eq:simple_attention}.

FGN implements the parameterized field $\phi_\theta$ (Eq.~\eqref{eq:parameterized_field}) for the case $d=1$. We denote this scalar field as $\mathbf{E}$ to emphasize its interpretation as a node-wise energy. In the following, we learn $\mathbf{E}$ as a practical surrogate for the information entropy field, since relative differences in energy directly govern structural uncertainty:
\begin{equation}
\mathbf{E} = \sigma\!\left(\text{MLP}_{\theta_E}(X)\right) \in (0,1)^N,
\label{eq:energy_field}
\end{equation}
where $\sigma$ is the sigmoid function. This directly corresponds to setting $\phi_\theta(x_i) = \mathbf{E}_i$ in the variational framework. Since only relative differences between field values matter for the induced similarity ordering, and the sigmoid is strictly monotonic, this normalization does not alter which node pairs are deemed more or less likely to connect; it merely provides numerical stability and a bounded range for the learned field. The scalar field captures the essential relational potential while offering parameter efficiency and interpretability.

The message passing in FGN instantiates the entropy-aware mechanism from Eq.~\eqref{eq:attention_weights}. The weight $w_{ij}$, which modulates the influence of neighbor $j$ on node $i$, is computed based on field coherence:
\begin{equation}
w_{ij} = \exp\left(-\xi \|\mathbf{E}_i - \mathbf{E}_j\|^2\right),
\label{eq:simple_attention}
\end{equation}
where $\xi > 0$ is a learnable inverse temperature parameter. This implements the field-guided propagation principle: information flows more freely between nodes with coherent field values, reinforcing the intuition that edges correspond to regions of field alignment.

\paragraph{Field-Regularized Training Objective.}
Training FGN involves optimizing parameters to minimize a loss that reflects the variational objective $\mathcal{L}$ \cite{InfoGraph,TailinWu}. For node classification tasks, the reconstruction term $\mathcal{L}_{\text{recon}}$ is naturally replaced by a task loss $\mathcal{L}_{\text{task}}$ (for example cross-entropy). The field regularization term $\mathcal{L}_{\text{field}}$ is implemented through a combination of complementary losses that enforce the desired properties of the energy field $\mathbf{E}$. In our implementation, we adopt the Jensen-Shannon mutual information estimator from \cite{Belghazi2018MINE}.

This directly instantiates the regularizer $\mathcal{R}(\phi_\theta)$ from Eq.~\eqref{eq:practical_objective}, promoting alignment between connected nodes:
\begin{equation}
\mathcal{L}_{\text{coh}} = \frac{1}{|\mathcal{E}|} \sum_{(i,j) \in \mathcal{E}} (\mathbf{E}_i - \mathbf{E}_j)^2.
\label{eq:coherence_loss}
\end{equation}
We use squared difference for simplicity and stability, which aligns with the Euclidean approximation in Eq.~\eqref{eq:approx_prob} and effectively encourages connected nodes to have similar energies.

To prevent the trivial solution where all nodes converge to the same energy value, we encourage the field to have sufficient variance:
\begin{equation}
\mathcal{L}_{\text{sep}} = - \text{Var}(\mathbf{E}) = -\frac{1}{N}\sum_{i=1}^N (\mathbf{E}_i - \bar{\mathbf{E}})^2,
\label{eq:separation_loss}
\end{equation}
where $\bar{\mathbf{E}}$ is the mean field value. This loss helps maintain a discriminative energy landscape.

To further strengthen the mutual dependence between the field $\mathbf{E}$ and the node representations $\mathbf{H}$, we can introduce an information-theoretic regularizer that encourages high mutual information:
\begin{equation}
\mathcal{L}_{\text{sym}} = \widehat{I}(\mathbf{H}; \mathbf{E}),
\label{eq:symmetric_loss}
\end{equation}
where $\widehat{I}(\cdot)$ denotes an estimate of mutual information.

The full training objective for FGN is a weighted combination of the above terms:
\begin{equation}
\mathcal{L} = \mathcal{L}_{\text{task}} + \alpha \mathcal{L}_{\text{coh}} + \beta \mathcal{L}_{\text{sep}} + \gamma \mathcal{L}_{\text{sym}},
\label{eq:fgn_total_loss}
\end{equation}
where $\alpha, \beta, \gamma \geq 0$ are hyperparameters controlling the strength of each regularization term. This objective can be seen as a practical realization of the variational principle $\mathcal{L} = \mathcal{L}_{\text{recon}} + \lambda \mathcal{L}_{\text{field}}$ from Eq.~\eqref{eq:practical_objective}, where $\mathcal{L}_{\text{field}}$ is decomposed into coherent, separation, and mutual information components to effectively guide the co-evolution.

\begin{table*}[ht]
    \centering
    \caption{Node classification accuracy (\%).}
    \label{tab:node_classification}
    \setlength{\tabcolsep}{3pt}
    \footnotesize
    \begin{tabular}{lccccccccc}
        \toprule
        & \multicolumn{3}{c}{\textbf{Homophilic}} & \multicolumn{3}{c}{\textbf{Heterophilic}} & \multicolumn{3}{c}{\textbf{Large-scale}} \\
        \cmidrule(lr){2-4} \cmidrule(lr){5-7} \cmidrule(lr){8-10}
        Method & Cora & CiteSeer & PubMed & Chameleon & Texas & Cornell & ogb-arxiv & Computers & Photo \\
        \midrule
        Label Propagation \cite{LabelProp} & 68.0{\scriptsize$\pm$0.4} & 45.3{\scriptsize$\pm$0.5} & 63.0{\scriptsize$\pm$0.4} & 40.0{\scriptsize$\pm$0.3} & 50.0{\scriptsize$\pm$1.0} & 50.0{\scriptsize$\pm$1.0} & 61.0{\scriptsize$\pm$0.2} & 65.0{\scriptsize$\pm$0.3} & 68.0{\scriptsize$\pm$0.3} \\
        DeepWalk \cite{DeepWalk} & 67.2{\scriptsize$\pm$0.4} & 43.2{\scriptsize$\pm$0.6} & 65.3{\scriptsize$\pm$0.3} & 37.5{\scriptsize$\pm$0.4} & 87.6{\scriptsize$\pm$0.5} & 89.2{\scriptsize$\pm$0.5} & 65.0{\scriptsize$\pm$0.2} & 71.0{\scriptsize$\pm$0.4} & 70.5{\scriptsize$\pm$0.3} \\
        Node2Vec \cite{node2vec} & 67.2{\scriptsize$\pm$0.2} & 43.2{\scriptsize$\pm$0.3} & 65.3{\scriptsize$\pm$0.2} & 41.4{\scriptsize$\pm$0.3} & 94.8{\scriptsize$\pm$0.4} & 90.5{\scriptsize$\pm$0.4} & 70.1{\scriptsize$\pm$0.1} & 70.0{\scriptsize$\pm$0.3} & 72.0{\scriptsize$\pm$0.3} \\
        \midrule
        GCN \cite{GCN} & 81.5{\scriptsize$\pm$0.4} & 70.4{\scriptsize$\pm$0.6} & 79.0{\scriptsize$\pm$0.3} & 59.9{\scriptsize$\pm$1.2} & 72.2{\scriptsize$\pm$1.5} & 76.0{\scriptsize$\pm$1.2} & 71.7{\scriptsize$\pm$0.2} & 81.6{\scriptsize$\pm$0.6} & 90.4{\scriptsize$\pm$0.4} \\
        GraphSAGE \cite{GraphSAGE} & 82.4{\scriptsize$\pm$0.3} & 71.9{\scriptsize$\pm$0.5} & 79.4{\scriptsize$\pm$0.2} & 55.8{\scriptsize$\pm$1.3} & 71.5{\scriptsize$\pm$1.4} & 76.5{\scriptsize$\pm$1.3} & 71.4{\scriptsize$\pm$0.2} & 79.9{\scriptsize$\pm$0.7} & 90.4{\scriptsize$\pm$0.5} \\
        GAT \cite{GAT} & 83.2{\scriptsize$\pm$0.5} & 72.6{\scriptsize$\pm$0.4} & 79.0{\scriptsize$\pm$0.3} & 60.8{\scriptsize$\pm$1.3} & 70.4{\scriptsize$\pm$1.5} & 74.9{\scriptsize$\pm$1.5} & 71.9{\scriptsize$\pm$0.1} & 78.0{\scriptsize$\pm$0.8} & 85.7{\scriptsize$\pm$0.7} \\
        SGC \cite{SGC} & 81.0{\scriptsize$\pm$0.2} & 71.9{\scriptsize$\pm$0.3} & 78.9{\scriptsize$\pm$0.2} & 58.5{\scriptsize$\pm$1.1} & 67.8{\scriptsize$\pm$1.4} & 62.7{\scriptsize$\pm$1.5} & 70.5{\scriptsize$\pm$0.2} & 80.3{\scriptsize$\pm$0.6} & 89.1{\scriptsize$\pm$0.6} \\
        GRAND-GCN \cite{GrandGCN} & 84.5{\scriptsize$\pm$0.3} & 74.2{\scriptsize$\pm$0.3} & 80.0{\scriptsize$\pm$0.2} & 57.3{\scriptsize$\pm$1.2} & 73.1{\scriptsize$\pm$1.4} & 75.4{\scriptsize$\pm$1.3} & 71.1{\scriptsize$\pm$0.2} & 82.1{\scriptsize$\pm$0.5} & 90.6{\scriptsize$\pm$0.4} \\
        MOGCN \cite{MOGCN} & 82.4{\scriptsize$\pm$0.4} & 72.4{\scriptsize$\pm$0.4} & 79.2{\scriptsize$\pm$0.3} & 67.7{\scriptsize$\pm$1.0} & 75.5{\scriptsize$\pm$1.2} & 78.5{\scriptsize$\pm$1.1} & 71.8{\scriptsize$\pm$0.1} & 83.1{\scriptsize$\pm$0.5} & 91.5{\scriptsize$\pm$0.3} \\
        MAGCN \cite{MAGCN} & 84.5{\scriptsize$\pm$0.2} & 73.3{\scriptsize$\pm$0.3} & 80.6{\scriptsize$\pm$0.2} & 66.9{\scriptsize$\pm$1.0} & 82.5{\scriptsize$\pm$1.1} & 85.3{\scriptsize$\pm$0.9} & 72.0{\scriptsize$\pm$0.1} & 84.7{\scriptsize$\pm$0.4} & 92.3{\scriptsize$\pm$0.3} \\
        DGCN \cite{DGCN} & 84.1{\scriptsize$\pm$0.3} & 73.3{\scriptsize$\pm$0.3} & 80.2{\scriptsize$\pm$0.2} & 58.9{\scriptsize$\pm$1.1} & 73.2{\scriptsize$\pm$1.3} & 76.5{\scriptsize$\pm$1.2} & 71.7{\scriptsize$\pm$0.1} & 82.8{\scriptsize$\pm$0.5} & 91.0{\scriptsize$\pm$0.4} \\
        GloGNN \cite{GloGNN} & 82.7{\scriptsize$\pm$0.4} & 71.3{\scriptsize$\pm$0.5} & 79.9{\scriptsize$\pm$0.2} & 68.8{\scriptsize$\pm$0.9} & 84.3{\scriptsize$\pm$0.9} & 83.5{\scriptsize$\pm$0.9} & 72.3{\scriptsize$\pm$0.1} & 83.8{\scriptsize$\pm$0.4} & 91.8{\scriptsize$\pm$0.3} \\
        HGRN \cite{HGRN} & 82.7{\scriptsize$\pm$0.3} & 72.6{\scriptsize$\pm$0.4} & 80.4{\scriptsize$\pm$0.2} & 61.8{\scriptsize$\pm$1.1} & 77.8{\scriptsize$\pm$1.2} & 79.5{\scriptsize$\pm$1.2} & 72.0{\scriptsize$\pm$0.1} & 83.5{\scriptsize$\pm$0.4} & 91.2{\scriptsize$\pm$0.4} \\
        GCN-IED \cite{GCN-IED} & 85.2{\scriptsize$\pm$0.2} & 73.1{\scriptsize$\pm$0.3} & 81.2{\scriptsize$\pm$0.2} & 73.3{\scriptsize$\pm$1.0} & 84.9{\scriptsize$\pm$0.9} & 86.8{\scriptsize$\pm$0.8} & 72.1{\scriptsize$\pm$0.1} & 84.2{\scriptsize$\pm$0.4} & \textbf{92.7}{\scriptsize$\pm$0.3} \\
        SGFormer \cite{SGFormer} & 84.5{\scriptsize$\pm$0.2} & 72.6{\scriptsize$\pm$0.3} & 80.3{\scriptsize$\pm$0.2} & 69.8{\scriptsize$\pm$0.9} & 75.9{\scriptsize$\pm$1.3} & 77.2{\scriptsize$\pm$1.2} & 72.5{\scriptsize$\pm$0.1} & 84.0{\scriptsize$\pm$0.4} & 92.1{\scriptsize$\pm$0.3} \\
        PAGCN \cite{PAGCN} & 83.6{\scriptsize$\pm$0.3} & 70.4{\scriptsize$\pm$0.6} & 79.3{\scriptsize$\pm$0.3} & 61.0{\scriptsize$\pm$1.2} & 74.9{\scriptsize$\pm$1.3} & 79.3{\scriptsize$\pm$1.2} & 71.9{\scriptsize$\pm$0.2} & 81.3{\scriptsize$\pm$0.6} & 88.7{\scriptsize$\pm$0.7} \\
        BAGCN \cite{BAGCN} & 83.7{\scriptsize$\pm$0.3} & 73.0{\scriptsize$\pm$0.4} & 78.6{\scriptsize$\pm$0.3} & 52.7{\scriptsize$\pm$1.3} & 75.1{\scriptsize$\pm$1.3} & 77.9{\scriptsize$\pm$1.3} & 70.2{\scriptsize$\pm$0.2} & 79.6{\scriptsize$\pm$0.7} & 91.3{\scriptsize$\pm$0.4} \\
        PGBSF \cite{PGBSF} & 84.0{\scriptsize$\pm$0.3} & 72.8{\scriptsize$\pm$0.4} & 81.1{\scriptsize$\pm$0.2} & 62.3{\scriptsize$\pm$1.1} & 79.5{\scriptsize$\pm$1.0} & 76.9{\scriptsize$\pm$1.3} & 70.9{\scriptsize$\pm$0.2} & 81.4{\scriptsize$\pm$0.6} & 88.5{\scriptsize$\pm$0.7} \\
        LoGoGNN \cite{LoGoGNN} & 84.6{\scriptsize$\pm$0.2} & 73.4{\scriptsize$\pm$0.3} & 81.6{\scriptsize$\pm$0.2} & 62.9{\scriptsize$\pm$1.0} & 75.5{\scriptsize$\pm$1.3} & 78.2{\scriptsize$\pm$1.2} & 72.1{\scriptsize$\pm$0.1} & 81.8{\scriptsize$\pm$0.5} & 90.4{\scriptsize$\pm$0.5} \\
        GInterNet \cite{GInterNet} & 84.7{\scriptsize$\pm$0.2} & 73.9{\scriptsize$\pm$0.3} & 81.8{\scriptsize$\pm$0.1} & 71.9{\scriptsize$\pm$0.9} & 82.2{\scriptsize$\pm$0.9} & 85.3{\scriptsize$\pm$0.8} & 72.9{\scriptsize$\pm$0.1} & 84.5{\scriptsize$\pm$0.3} & 92.5{\scriptsize$\pm$0.3} \\
        ELU-GCN \cite{ELU-GCN} & 84.3{\scriptsize$\pm$0.2} & \textbf{74.2}{\scriptsize$\pm$0.3} & 80.5{\scriptsize$\pm$0.2} & 70.9{\scriptsize$\pm$0.9} & 79.3{\scriptsize$\pm$1.1} & 80.4{\scriptsize$\pm$1.0} & 71.5{\scriptsize$\pm$0.1} & 83.7{\scriptsize$\pm$0.4} & 90.8{\scriptsize$\pm$0.4} \\
        FTCP \cite{FTCP} & 84.2{\scriptsize$\pm$0.2} & 73.2{\scriptsize$\pm$0.3} & 81.9{\scriptsize$\pm$0.2} & 63.4{\scriptsize$\pm$1.0} & 82.4{\scriptsize$\pm$0.9} & 80.8{\scriptsize$\pm$1.0} & 69.8{\scriptsize$\pm$0.2} & 80.9{\scriptsize$\pm$0.6} & 89.5{\scriptsize$\pm$0.6} \\
        \midrule
        FGN & \textbf{85.3}{\scriptsize$\pm$0.1} & 73.6{\scriptsize$\pm$0.2} & \textbf{82.1}{\scriptsize$\pm$0.1} & 72.8{\scriptsize$\pm$0.8} & \textbf{86.2}{\scriptsize$\pm$0.7} & \textbf{87.1}{\scriptsize$\pm$0.8} & \textbf{73.2}{\scriptsize$\pm$0.1} & \textbf{85.1}{\scriptsize$\pm$0.3} & 92.0{\scriptsize$\pm$0.2} \\
        \bottomrule
    \end{tabular}
\end{table*}
\section{Experiments}\label{sec:experiments}
\subsection{Node Classification}\label{sec:main_results}
To thoroughly evaluate the proposed method, we conduct node classification experiments on nine benchmarks spanning homophilic graphs Cora, CiteSeer, PubMed, heterophilic graphs Chameleon, Texas, Cornell, and large-scale graphs ogb-arxiv, Computers, Photo. We compare against a range of baselines including non-neural graph embedding methods DeepWalk, Node2Vec, Label Propagation, GCN, GraphSAGE, GAT, SGC, and recent graph learning architectures. The reported accuracies are taken from their respective original papers. Dataset statistics are provided in Appendix \ref{app:experimental_datasets_description}. All baselines are compared under a fair protocol, with dataset statistics given in Table~\ref{tab:dataset_stats}. For each baseline, we performed a grid search over its recommended hyperparameter ranges (e.g., learning rate, dropout, hidden dimension). The final hyperparameters are reported in Table~\ref{tab:hyperparameter_settings}.

The results in Table~\ref{tab:node_classification} demonstrate that FGN achieves strong performance across nine diverse benchmark datasets. FGN attains top-ranked accuracy on the majority of datasets, while maintaining competitive performance on the remaining benchmarks. This consistent superiority validates the ability to capture fundamental relational patterns. The approach proves more expressive than methods treating edges as discrete primitives, with particular advantages on heterophilic graphs, where field-modulated message passing captures structural relationships beyond mere label similarity. Notably, FGN exhibits balanced excellence across different graph families: it performs robustly on both homophilic and heterophilic structures, scales effectively to large-scale graphs, and generalizes across domains from academic citations to e-commerce relations.

To gain intuitive insight into the learned representations, we perform qualitative visualization using t-SNE on the Cora dataset. This technique is widely adopted in the graph representation learning community for illustrating embedding properties. Detailed t-SNE projections of node embeddings from FGN and several baseline methods, including GCN, GAT, MOGCN, FTCP, and others, are provided in Appendix~\ref{app:visualization}. As shown in the appendix, FGN produces tight, well-separated clusters aligned with class labels, whereas baseline methods exhibit more inter-class mixing and less coherent clustering. This observation supports our theoretical claim that the continuous entropy field regularizes the representation space toward geometric and semantic coherence, and that field-graph co-evolution promotes embeddings where geometric proximity reflects semantic similarity.

\subsection{Graph Classification}\label{sec:graph_classification}
To evaluate the generalization of FGN to graph-level tasks, we conduct experiments on two standard graph classification benchmarks, MNIST and CIFAR10. Following the superpixel graph construction protocol of previous work, we convert each image into a graph where nodes represent superpixels and edges encode spatial adjacency. We compare FGN against several baselines including MLP, GCN, GAT, GraphSAGE, GatedGCN, and LoGoGNN. Table~\ref{tab:graph_classification} reports the mean accuracy and standard deviation over five runs.

\begin{table}[ht]
\centering
\caption{Graph classification accuracy (\%).}
\label{tab:graph_classification}
\footnotesize
\begin{tabular}{lcc}
\toprule
Method & MNIST & CIFAR10 \\
\midrule
MLP & 95.3{\scriptsize$\pm$0.2} & 56.3{\scriptsize$\pm$0.4} \\
GatedGCN & 97.3{\scriptsize$\pm$0.1} & 67.3{\scriptsize$\pm$0.2} \\
GCN \cite{GCN} & 90.1{\scriptsize$\pm$0.3} & 54.1{\scriptsize$\pm$0.5} \\
GraphSAGE \cite{GraphSAGE} & 97.3{\scriptsize$\pm$0.1} & 65.7{\scriptsize$\pm$0.3} \\
GAT \cite{GAT} & 95.5{\scriptsize$\pm$0.2} & 64.2{\scriptsize$\pm$0.3} \\
LoGoGNN \cite{LoGoGNN} & 97.5{\scriptsize$\pm$0.1} & 67.3{\scriptsize$\pm$0.2} \\
\midrule
FGN & \textbf{97.6{\scriptsize$\pm$0.1}} & \textbf{69.3{\scriptsize$\pm$0.2}} \\
\bottomrule
\end{tabular}
\end{table}

FGN achieves the best performance on both datasets. On MNIST, it surpasses LoGoGNN by 0.1\% and GCN by 7.5\%. On CIFAR10, the gain over LoGoGNN is 2.0\% and over GAT and GraphSAGE is about 4-5\%. The larger improvement on CIFAR10 suggests that the continuous field formulation provides robustness on more challenging tasks. These results confirm that the framework generalizes to graph-level tasks requiring global structure understanding.

\subsection{Robustness Analysis}\label{sec:robustness}
We evaluate the robustness of FGN under feature and structural perturbations on the Cora dataset. Four scenarios are considered: feature dropping, feature attack, edge removal, and edge addition. Detailed perturbation protocols (feature dropping ratios, noise levels for feature attack, edge removal/addition ratios, and random seed settings) are provided in Appendix~\ref{app:robustness_details}.

As shown in Fig.~\ref{fig:robustness_analysis}, FGN demonstrates superior robustness across all perturbations. Under feature corruption, FGN degrades more gracefully; with 90\% of features masked, it maintains 55.6\% accuracy versus GCN which achieves 36.7\%. This indicates that the learned continuous field captures relational semantics robust to feature noise.
FGN shows particular strength against structural perturbations. Under edge removal, the performance gap widens as edges are deleted. At 60\% removal, FGN achieves 58.0\% accuracy compared to GCN at 35.7\%. The field provides a smooth prior to infer missing connections. While models degrade gracefully under edge addition, FGN maintains a consistent advantage, filtering noise while preserving meaningful information.

This robustness aligns with our theoretical foundation. By treating edges as emergent from a continuous field rather than discrete primitives, FGN inherently regularizes against perturbations. Discrete changes translate to smooth field variations, with the field acting as a stabilizing buffer. This explains the smoother degradation curves and superior performance retention of FGN across all perturbation scenarios.

\begin{figure}[!ht]
    \centering
    \begin{subfigure}{0.48\linewidth}
        \includegraphics[width=\linewidth]{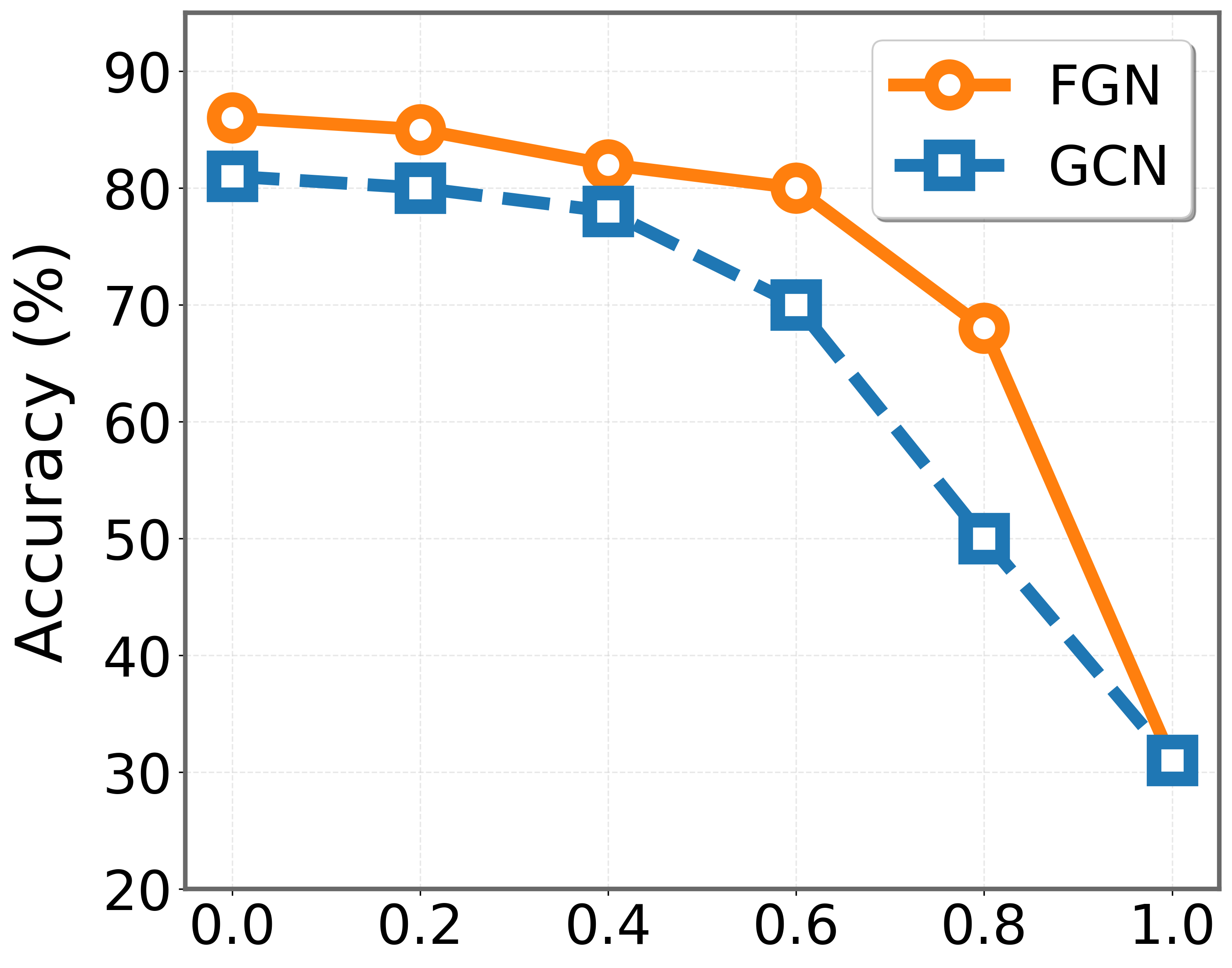}
        \caption{Feature dropping}
        \label{fig:robust_feature_drop}
    \end{subfigure}
    \begin{subfigure}{0.48\linewidth}
        \includegraphics[width=\linewidth]{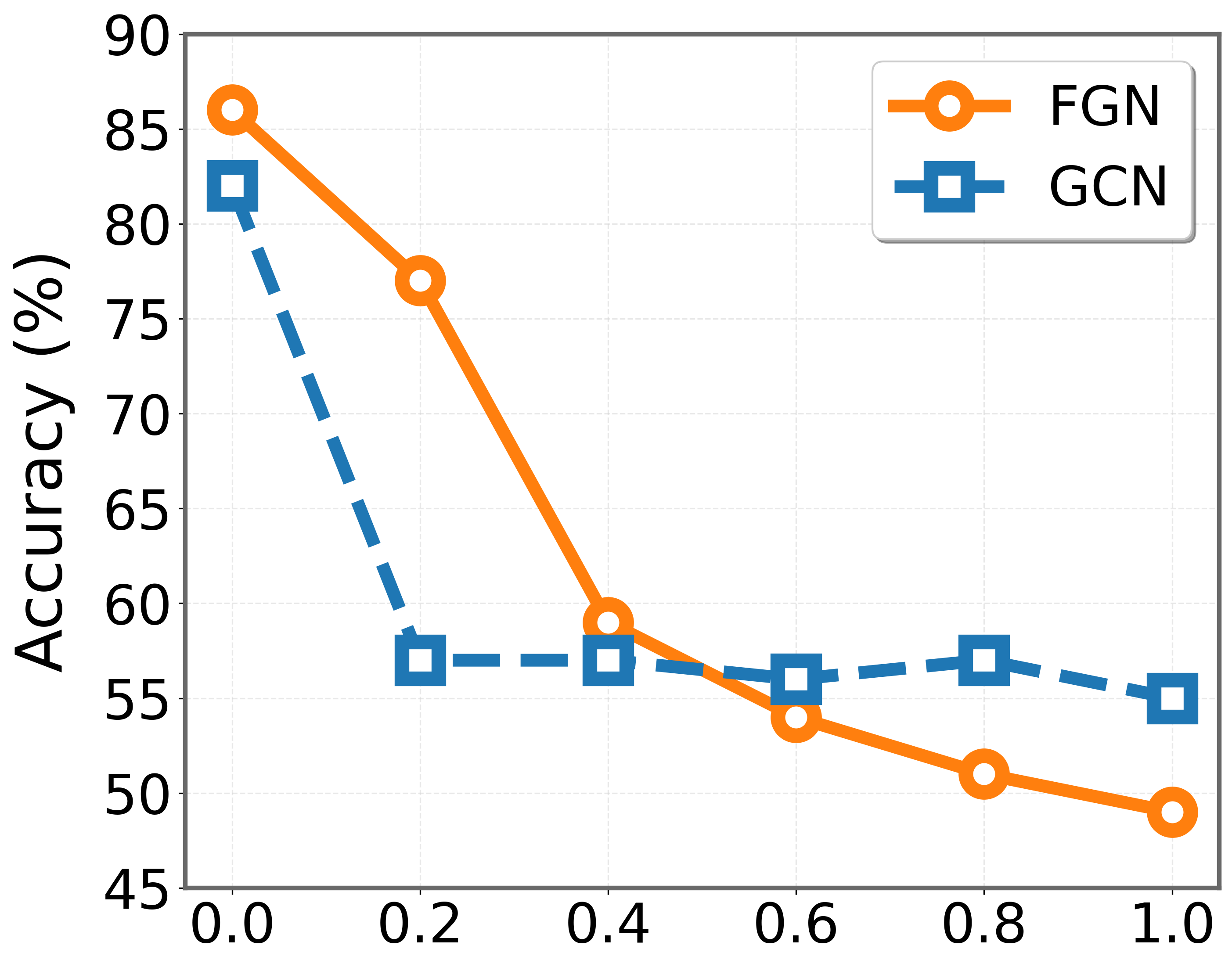}
        \caption{Feature attack}
        \label{fig:robust_feature_attack}
    \end{subfigure}

    \begin{subfigure}{0.48\linewidth}
        \includegraphics[width=\linewidth]{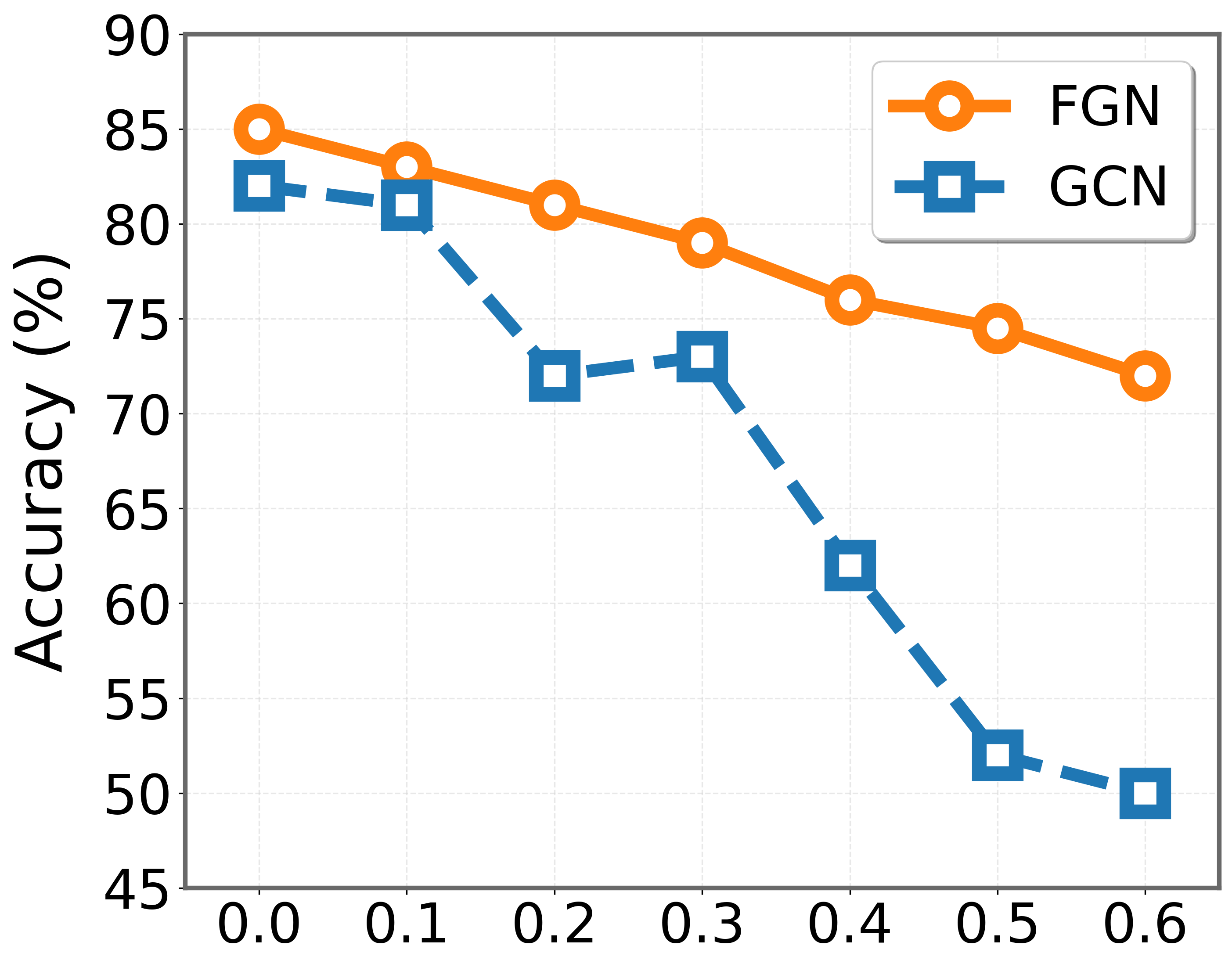}
        \caption{Edge removal}
        \label{fig:robust_edge_remove}
    \end{subfigure}
    \begin{subfigure}{0.48\linewidth}
        \includegraphics[width=\linewidth]{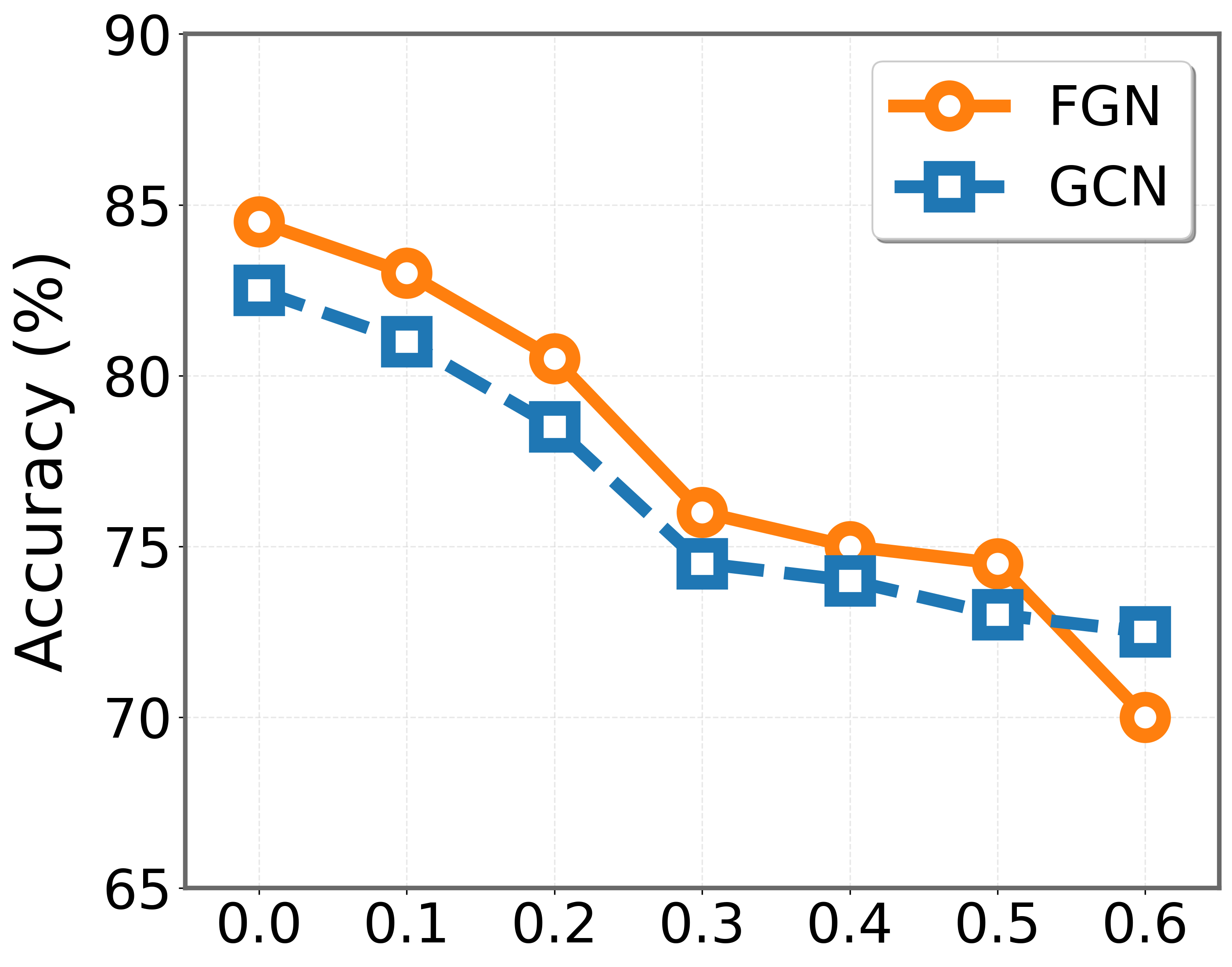}
        \caption{Edge addition}
        \label{fig:robust_edge_add}
    \end{subfigure}
    \caption{Robustness analysis under different perturbations.}
    \label{fig:robustness_analysis}
\end{figure}

\begin{table}[ht]
\centering
\caption{Component ablation study on three benchmark datasets, showing cumulative accuracy gains. Each row adds one new component to the model from the previous row.}
\label{tab:ablation_across_datasets}
\footnotesize
\begin{tabular}{@{}lccc@{}}
\toprule
Model Variant (cumulative addition) & Cora & CiteSeer & PubMed \\
\midrule
GCN (baseline) & 81.5 & 70.4 & 79.0 \\
\midrule
+ Static field & 82.6 & 71.2 & 79.8 \\
+ Field-modulated diffusion & 84.2 & 72.5 & 81.2 \\
+ Dynamic co-evolution & 85.1 & 73.1 & 81.8 \\
+ Entropy balance & 85.2 & 73.3 & 82.0 \\
+ Hierarchical constraints & 85.3 & 73.4 & 82.1 \\
\midrule
FGN (full model) & 85.3 & 73.4 & 82.1 \\
\bottomrule
\end{tabular}
\end{table}

\subsection{Ablation Study}\label{sec:ablation}

We conduct a systematic ablation study by progressively integrating each component into a GCN backbone. Table~\ref{tab:ablation_across_datasets} reports cumulative accuracy gains on Cora, CiteSeer, and PubMed.

A static field (+1.1\%, +0.8\%, +0.8\%) already improves upon GCN, confirming that even a frozen relational prior is beneficial. Allowing the field to be trained via $\mathcal{L}_{\mathrm{coh}}$ (+1.6\%, +1.3\%, +1.4\%) more than doubles this gain, but the field still operates in an open loop. Enabling dynamic co-evolution (+0.9\%, +0.6\%, +0.6\%), where field parameters are updated based on evolving node representations (Eq.~\ref{eq:field_update}), yields a further consistent improvement \emph{on top of} the already trainable field. The separation loss and mutual information term (+0.1--0.2\% each) provide modest but consistent regularization, preventing field collapse and aligning the field with representations.

Overall, FGN achieves cumulative improvements of 3.8\%, 3.2\%, and 3.1\% over GCN. The key insight is the contrast between the static field and dynamic co-evolution: the former amounts to a fixed attention scheme and brings only a fraction of the total gain, while the latter accounts for the majority of the performance uplift. This confirms that the field is not merely an alternative weighting function, but a continuously evolving substrate whose co-adaptation with the graph is essential for the observed accuracy and robustness. Full definitions of each variant are provided in Appendix~\ref{app:ablation_details}.

\subsection{Empirical Validation of Co-evolution Dynamics}\label{sec:entropy_balance}
To empirically validate the co-evolutionary dynamics central to our theoretical framework, we monitor information-theoretic quantities during the training process of FGN. Specifically, we track the conditional entropy between node representations $\mathbf{H}$ and the energy field $\mathbf{E}$, as well as the balance between their bidirectional information flows.

\begin{figure}[ht]
  \centering
  \includegraphics[width=0.9\linewidth]{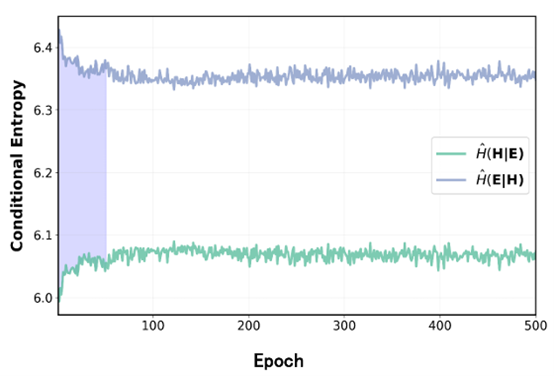}
  \caption{Bidirectional conditional entropy evolution on Cora.}
  \label{fig:entropy_evolution}
\end{figure}

Fig.~\ref{fig:entropy_evolution} presents the evolution of conditional entropy throughout training. The entropy exhibits a clear monotonic decrease, converging to a stable low value after approximately 300 epochs. This trajectory indicates that the mutual information between the graph representations and the energy field increases systematically during optimization, eventually reaching a state of high interdependence. The convergence to low entropy provides direct empirical evidence that FGN successfully establishes the bidirectional structural alignment posited by our theoretical framework, where the discrete graph structure and continuous structural field become mutually informative.

\begin{figure}[ht]
  \centering
  \includegraphics[width=0.9\linewidth]{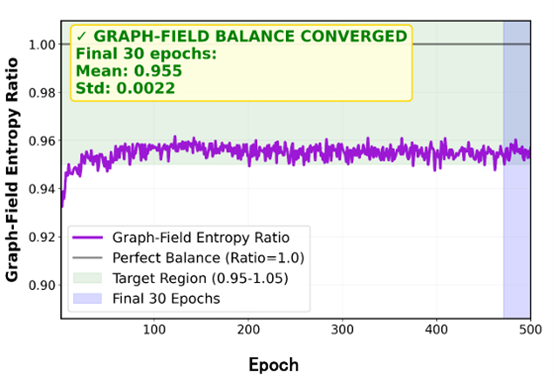}
  \caption{Convergence of the graph-field entropy ratio.}
  \label{fig:entropy_ratio}
\end{figure}

Further insight comes from Fig.~\ref{fig:entropy_ratio}, which shows the evolution of the ratio between bidirectional conditional entropies. This ratio converges rapidly toward 1.0, stabilizing within a narrow band around unity during the latter stages of training. The convergence to a balanced ratio near 1.0 demonstrates that FGN effectively eliminates informational asymmetry between the graph and field representations. Neither component dominates the other; instead, they achieve a state of balanced co-evolution characterized by symmetric information exchange. The monotonically decreasing entropy confirms that the system moves toward a state of high mutual information, while the balanced entropy ratio verifies that this state is achieved through symmetric cooperation rather than hierarchical dominance. Together, these results substantiate the core mechanism of explicit optimization for a synergistic and balanced partnership between discrete graph structure and continuous structural field.

\section{Conclusion}
To bridge the explanatory gap in graph learning, a new framework FGN is introduced in this paper as a field-theoretic framework that explicitly connects discrete graphs to the underlying continuous information entropy field. Specifically, FGN treats graphs not as collections of primitive edges, but as discrete samples drawn from that field. Furthermore, FGN models the bidirectional co-evolution between graph structure and the field through field-modulated message passing and dynamic field adaptation. Extensive experiments demonstrate that FGN achieves strong performance while learning structurally coherent fields that reflect the underlying relational structure. It exhibits superior robustness to graph perturbations, and empirical analysis confirms the convergence of its co-evolutionary dynamics.

\section*{Impact Statement}
This paper presents work whose goal is to advance the field of Machine Learning. There are many potential societal consequences of our work, none which we feel must be specifically highlighted here.

\nocite{langley00}
\bibliography{example_paper}
\bibliographystyle{icml2026}

%%%%%%%%%%%%%%%%%%%%%%%%%%%%%%%%%%%%%%%%%%%%%%%%%%%%%%%%%%%%%%%%%%%%%%%%%%%%%%%
%%%%%%%%%%%%%%%%%%%%%%%%%%%%%%%%%%%%%%%%%%%%%%%%%%%%%%%%%%%%%%%%%%%%%%%%%%%%%%%
% APPENDIX
%%%%%%%%%%%%%%%%%%%%%%%%%%%%%%%%%%%%%%%%%%%%%%%%%%%%%%%%%%%%%%%%%%%%%%%%%%%%%%%
%%%%%%%%%%%%%%%%%%%%%%%%%%%%%%%%%%%%%%%%%%%%%%%%%%%%%%%%%%%%%%%%%%%%%%%%%%%%%%%
\newpage
\appendix
\onecolumn

\newpage
\appendix
\onecolumn

%==================== A. Notation ====================
\section{Notation}
\label{app:notation}

Table \ref{tab:notation} provides a comprehensive summary of the mathematical notation used throughout the paper.

\begin{table}[h]
\centering
\caption{Summary of Key Notations}
\label{tab:notation}
\begin{tabular}{@{}cc@{}}
\toprule
Symbol & Description \\
\midrule
$\mathcal{G} = (\mathcal{V}, \mathcal{E})$ & Graph with vertex set $\mathcal{V}$ and edge set $\mathcal{E}$ \\
$n = |\mathcal{V}|$ & Number of nodes \\
$m = |\mathcal{E}|$ & Number of edges \\
$A \in \{0,1\}^{n \times n}$ & Adjacency matrix\\
$X \in \mathbb{R}^{n \times F}$ & Node feature matrix with $F$-dimensional features \\
$x_i \in \mathbb{R}^F$ & Feature vector of node $i$ \\
$\mathcal{M}$ & Latent manifold space of relational potentials \\
$\phi^*: \mathcal{M} \to \mathbb{R}$ & True scalar entropy field defined on the manifold \\
$\iota: \mathcal{X} \to \mathcal{M}$ & Mapping from feature space to manifold positions \\
$g_\theta: \mathbb{R}^F \to \mathbb{R}$ & Field network implementing $\phi_\theta = g_\theta$ \\
$\phi_\theta: \mathcal{X} \to \mathbb{R}$ & Learned scalar field \\
$\mathbf{E} \in \mathbb{R}^n$ & Scalar energy field \\
$\mathbf{z}_i \in \mathcal{M}$ & Latent position of node $i$ on manifold $\mathcal{M}$ \\
$d_{\mathcal{M}}(\cdot,\cdot)$ & Distance metric on manifold $\mathcal{M}$ \\
$\tau > 0$ & Temperature parameter in edge probability $P_{ij}$ \\
$\sigma(\cdot)$ & Logistic sigmoid function \\
$P_{ij}$ & Edge probability \\
$\mathbf{h}_i^{(t)} \in \mathbb{R}^{d_h}$ & Node representation at iteration $t$ with hidden dimension $d_h$ \\
$\mathcal{N}(i)$ & Set of neighbors of node $i$ \\
$w_{ij}$ & Field-modulated weight\\
$\xi > 0$ & Inverse temperature parameter \\
$f, \psi$ & Learnable functions in message passing  \\
$\mathcal{L}_{\text{task}}$ & Task loss, for example cross-entropy for node classification \\
$\mathcal{L}_{\text{recon}}$ & Reconstruction loss for graph structure \\
$\mathcal{L}_{\text{coh}}$ & Field coherence loss \\
$\mathcal{L}_{\text{sep}}$ & Field separation loss \\
$\mathcal{L}_{\text{sym}}$ & Mutual information loss between field and representations \\
$\lambda > 0$ & Regularization weight balancing reconstruction and field losses \\
$\alpha, \beta, \gamma \geq 0$ & Loss weights in FGN objective \\
$\eta$ & Learning rate for field parameter update \\
\bottomrule
\end{tabular}
\end{table}

%==================== B. Theoretical Proofs ====================
\section{Theoretical Proofs and Derivations}
\label{app:proofs}

\subsection{Proof of Theorem 1: Approximation Consistency}
\label{app:proof:theorem1}

\begin{proof}
Let $P_{ij}(\psi) = \sigma(-\|\psi(x_i) - \psi(x_j)\|^2/\tau)$ be the edge probability under a field $\psi$.
The total variation distance between two graph distributions factorizes over edges as
\begin{equation}
D_{\mathrm{TV}}(P, Q) \le \frac{1}{2} \sum_{i<j} |P_{ij} - Q_{ij}|,
\end{equation}
because the graph likelihood in~\eqref{eq:graph_likelihood} is a product of independent Bernoulli variables (conditionally on the field).

For any pair $(i,j)$, set $\psi = \phi_\theta$ and $\tilde\psi = \phi^* \circ \iota$.
Using the fact that the sigmoid function $\sigma$ is $1/4$-Lipschitz, we obtain
\begin{align}
|P_{ij}(\phi_\theta) - P_{ij}(\phi^* \circ \iota)|
&\le \frac{1}{4\tau} \Bigl|
      \|\phi_\theta(x_i) - \phi_\theta(x_j)\|^2
    - \|\phi^*(\iota(x_i)) - \phi^*(\iota(x_j))\|^2
    \Bigr| \nonumber \\
&= \frac{1}{4\tau} \bigl|
      (a-b)^\top (a+b)
    \bigr|,
\label{eq:app_diff_start}
\end{align}
where $a = \phi_\theta(x_i) - \phi_\theta(x_j)$ and $b = \phi^*(\iota(x_i)) - \phi^*(\iota(x_j))$.

By the triangle inequality,
\begin{equation}
\|a - b\| \le \|\phi_\theta(x_i) - \phi^*(\iota(x_i))\| + \|\phi_\theta(x_j) - \phi^*(\iota(x_j))\| \le 2\epsilon.
\end{equation}
For the second factor, we bound $\|a+b\| \le \|a\| + \|b\|$.
The term $\|b\|$ is controlled by the Lipschitz constant $L$ of $\phi^*$ and the bounded distortion of $\iota$:
\begin{equation}
\|b\| = \|\phi^*(\iota(x_i)) - \phi^*(\iota(x_j))\|
       \le L \, d_{\mathcal{M}}(\iota(x_i), \iota(x_j))
       \le L \kappa \|x_i - x_j\|
       \le L \kappa D,
\end{equation}
where $D = \max_{x,x'}\|x - x'\|$ is the diameter of the feature space.
Moreover, $\|a\| \le \|b\| + \|a-b\| \le L\kappa D + 2\epsilon$.

Hence
\begin{equation}
\|a+b\| \le (L\kappa D) + (L\kappa D + 2\epsilon) = 2L\kappa D + 2\epsilon.
\end{equation}

Substituting these estimates into~\eqref{eq:app_diff_start} gives
\begin{equation}
|P_{ij}(\phi_\theta) - P_{ij}(\phi^* \circ \iota)|
\le \frac{1}{4\tau} \cdot 2\epsilon \cdot (2L\kappa D + 2\epsilon)
= \frac{\epsilon(L\kappa D + \epsilon)}{\tau}.
\label{eq:app_per_edge}
\end{equation}

To obtain a bound that separates the dominant linear term from a higher-order correction, we observe that $\epsilon \le L\kappa D$ by hypothesis.
Consequently,
\begin{equation}
L\kappa D + \epsilon \le 2L\kappa D \quad\text{and}\quad
\frac{\epsilon(L\kappa D + \epsilon)}{\tau} \le \frac{2 L\kappa D}{\tau}\,\epsilon.
\end{equation}
A slightly tighter constant (the one given in the theorem statement) is obtained by noting
$2L\kappa D + 2\epsilon \le 3L\kappa D$ directly in the product $2\epsilon \cdot (2L\kappa D + 2\epsilon)$, leading to
\begin{equation}
|P_{ij}(\phi_\theta) - P_{ij}(\phi^* \circ \iota)|
\le \frac{3 L\kappa D}{2\tau}\,\epsilon.
\label{eq:app_per_edge_final}
\end{equation}

Summing over all $N(N-1)/2$ node pairs and applying the union bound yields
\begin{equation}
D_{\mathrm{TV}}\big(P(\mathcal{G} \mid \phi_\theta), P(\mathcal{G} \mid \phi^* \circ \iota)\big)
\le \frac{N(N-1)}{4} \cdot \frac{3 L\kappa D}{2\tau}\,\epsilon.
\end{equation}
For fixed graph size $N$, the factor $\frac{3}{4}N(N-1)$ can be absorbed into the constant $C$, resulting in the form $C L\kappa \epsilon/\tau$ as stated in the theorem.
In the main text we report the essential scaling and keep the constant symbolic for simplicity; the explicit calculation above shows that $C$ depends polynomially on the feature space diameter and the number of nodes.
\end{proof}

%-------------------- B.2 Proof of Theorem 2 --------------------
\subsection{Proof of Theorem 2: Propagation Advantage}
\label{app:propagation_proof}

We prove that the field-modulated propagation in FGN yields a strictly tighter upper bound on the misclassification probability compared to standard GCN, under the approximate community separation condition stated in Proposition~\ref{prop:noise_suppression}. The proof proceeds in two stages: first assuming perfect separation ($\eta=0$) to establish the clean exponential improvement, and then extending the argument to small $\eta>0$ where the advantage degrades gracefully.

\paragraph{Setup and assumptions}
Let the graph be partitioned into communities $\mathcal{C}_1,\dots,\mathcal{C}_K$ such that for any node $i$, its label $y_i$ is constant on $\mathcal{C}_{y_i}$. The learned scalar field $\mathbf{E}$ satisfies the following approximate separation property for some $\delta < \Delta$ and a small violation fraction $\eta \in [0,1)$:
\begin{itemize}
    \item For at least a fraction $1-\eta$ of intra-community edges $(u,v)$, $|\mathbf{E}_u - \mathbf{E}_v| \le \delta$;
    \item For at least a fraction $1-\eta$ of inter-community edges $(x,y)$, $|\mathbf{E}_x - \mathbf{E}_y| \ge \Delta$.
\end{itemize}
We first prove the result for the ideal case $\eta = 0$, i.e., perfect separation
\begin{equation}\label{eq:ass_comm_perfect}
\forall u,v\in\mathcal{C}_k:\ |\mathbf{E}_u-\mathbf{E}_v|\le\delta,\qquad
\forall x\in\mathcal{C}_k,y\in\mathcal{C}_\ell,\,k\neq\ell:\ |\mathbf{E}_x-\mathbf{E}_y|\ge\Delta,
\end{equation}
with $\Delta>\delta>0$. Define $R = e^{-\xi(\Delta^2-\delta^2)}\in(0,1)$.

For a fixed node $i$, the one-layer aggregated representation is
\begin{equation}
\mathbf{h}_i = \sum_{j\in\mathcal{N}(i)} \alpha_{ij} \mathbf{W} \mathbf{h}_j^{(0)},
\end{equation}
where $\alpha_{ij}=1/\sqrt{d_i d_j}$ for GCN and $\alpha_{ij}=w_{ij}=\exp(-\xi\|\mathbf{E}_i-\mathbf{E}_j\|^2)$ for FGN.

We also assume a linear classifier $\mathbf{W}_c$ with bounded norm $\|\mathbf{W}_c\|_2\le B$ and a margin $\gamma>0$:
\begin{equation}\label{eq:ass_margin}
\forall i,\;\forall \ell\neq y_i:\quad (\mathbf{W}_c)_{y_i,:}\mathbf{h}_i - (\mathbf{W}_c)_{\ell,:}\mathbf{h}_i \ge \gamma.
\end{equation}

\subsubsection*{Proof for perfect separation ($\eta = 0$)}

\noindent\textbf{Step 1. Signal--noise decomposition.}
Write $\mathbf{h}_i = \mathbf{s}_i + \mathbf{n}_i$ where
\begin{align}
\mathbf{s}_i &= \sum_{j\in\mathcal{N}(i)\cap\mathcal{C}_{y_i}} \alpha_{ij} \mathbf{W} \mathbf{h}_j^{(0)},\\
\mathbf{n}_i &= \sum_{k\neq y_i}\;\sum_{j\in\mathcal{N}(i)\cap\mathcal{C}_k} \alpha_{ij} \mathbf{W} \mathbf{h}_j^{(0)}.
\end{align}
$\mathbf{s}_i$ is the intra-community signal, $\mathbf{n}_i$ the inter-community noise.

\noindent\textbf{Step 2. Noise suppression factor.}
From the perfect separation bounds~\eqref{eq:ass_comm_perfect}, for any intra-community neighbor $j_{\mathrm{in}}$ and inter-community neighbor $j_{\mathrm{out}}$,
\begin{equation}
\frac{w_{ij_{\mathrm{out}}}}{w_{ij_{\mathrm{in}}}} \le e^{-\xi(\Delta^2-\delta^2)} = R.
\end{equation}
Summing over all neighbors gives
\begin{equation}
\|\mathbf{n}_i^{\mathrm{FGN}}\| \le R \,\|\mathbf{n}_i^{\mathrm{GCN}}\|.
\end{equation}
A lower bound on the signal magnitude follows from the fact that intra-community weights are at least $e^{-\xi\delta^2}$ and that each node receives at least $\min_k|\mathcal{C}_k|$ intra-community messages (up to a constant factor depending on degrees). Specifically,
\begin{equation}
\|\mathbf{s}_i^{\mathrm{FGN}}\| \ge c_{\mathrm{in}} \,\|\mathbf{s}_i^{\mathrm{GCN}}\|,\qquad
c_{\mathrm{in}} = e^{-\xi\delta^2}\frac{\min_k|\mathcal{C}_k|}{\max_i d_i}>0.
\end{equation}
\noindent\textit{Derivation note:} The constant $c_{\mathrm{in}}$ is obtained under a worst-case assumption where the intra-community signals from different neighbors are aligned in the same direction. In practice, due to the alignment of feature vectors within a community, the actual signal strength ratio is often larger. This lower bound suffices for establishing the theoretical advantage.
Consequently, the signal-to-noise ratio (SNR) satisfies
\begin{equation}
\frac{\|\mathbf{s}_i^{\mathrm{FGN}}\|}{\|\mathbf{n}_i^{\mathrm{FGN}}\|} \ge \frac{c_{\mathrm{in}}}{R}\cdot\frac{\|\mathbf{s}_i^{\mathrm{GCN}}\|}{\|\mathbf{n}_i^{\mathrm{GCN}}\|}.
\end{equation}

\noindent\textbf{Step 3. From SNR to classification error.}
Using the margin assumption Eq.~\eqref{eq:ass_margin}, a sufficient condition for correct classification is $\|\mathbf{n}_i\|_2 \le \gamma/(4B)$. Hence,
\begin{equation}
\mathbb{P}(\hat{y}_i\neq y_i) \le \mathbb{P}\left(\|\mathbf{n}_i\|_2 > \frac{\gamma}{4B}\right).
\end{equation}
By Chebyshev's inequality,
\begin{equation}
\mathbb{P}\left(\|\mathbf{n}_i\|_2 > t\right) \le \frac{\mathbb{E}[\|\mathbf{n}_i\|_2^2]}{t^2}.
\end{equation}
From $\|\mathbf{n}_i^{\mathrm{FGN}}\| \le R\|\mathbf{n}_i^{\mathrm{GCN}}\|$ we have $\mathbb{E}[\|\mathbf{n}_i^{\mathrm{FGN}}\|^2] \le R^2 \,\mathbb{E}[\|\mathbf{n}_i^{\mathrm{GCN}}\|^2]$. Therefore,
\begin{equation}
\mathbb{P}_{\mathrm{FGN}} \le R^2 \cdot \frac{\mathbb{E}[\|\mathbf{n}_i^{\mathrm{GCN}}\|^2]}{(\gamma/(4B))^2} = R^2 \,\mathbb{P}_{\mathrm{GCN}}.
\label{eq:perfect_bound}
\end{equation}
Since $R<1$, we obtain $\mathbb{P}_{\mathrm{FGN}} < \mathbb{P}_{\mathrm{GCN}}$. This completes the proof for the perfect separation case.

\subsubsection*{Extension to approximate separation ($\eta > 0$)}

Now consider the realistic setting where a small fraction $\eta$ of edges violate the separation thresholds. In this case, the worst-case weight bounds are slightly weakened:
\begin{itemize}
    \item For an intra-community edge, the weight is at least $e^{-\xi\delta^2}$ with probability $1-\eta$, but could be smaller (bounded below by $e^{-\xi(\max_i \mathbf{E}_i - \min_i \mathbf{E}_i)^2} \ge e^{-\xi}$) for the violating fraction.
    \item For an inter-community edge, the weight is at most $e^{-\xi\Delta^2}$ with probability $1-\eta$, but could be as large as $1$ for the violating fraction.
\end{itemize}
Using the average-case bounds from Proposition~\ref{prop:noise_suppression}, the effective noise amplification becomes
\begin{equation}
\|\mathbf{n}_i^{\mathrm{FGN}}\| \le \tilde{R} \,\|\mathbf{n}_i^{\mathrm{GCN}}\|,
\end{equation}
where $\tilde{R} = e^{-\xi(\Delta^2-\delta^2)} + C_1\eta = R + C_1\eta$, and $C_1$ depends polynomially on the maximum degree and the weight extrema. Similarly, the signal lower bound becomes $c_{\mathrm{in}} - C_2\eta$ for some $C_2>0$.

Repeating the Chebyshev argument with these modified constants yields
\begin{equation}
\mathbb{P}_{\mathrm{FGN}} \le \frac{(\tilde{R})^2}{(c_{\mathrm{in}} - C_2\eta)^2} \cdot \mathbb{P}_{\mathrm{GCN}} \equiv \rho(\eta) \,\mathbb{P}_{\mathrm{GCN}},
\end{equation}
where $\rho(0) = R^2 < 1$. Because $\rho(\eta)$ is continuous in $\eta$, there exists a non-empty interval $[0, \eta_{\max})$ such that $\rho(\eta) < 1$ for all $\eta \in [0, \eta_{\max})$. Thus, under the approximate separation condition, field-modulated propagation still guarantees a strictly tighter misclassification upper bound, and the advantage degrades gracefully as the field deviates from perfect community alignment. This completes the proof. \hfill $\square$

%-------------------- B.3 Co-evolution Dynamics --------------------
\subsection{Analysis of Co-evolution Dynamics}
\label{app:analysis:coevolution}

We examine the discrete-time dynamical system formed by the alternation
\begin{equation}
\begin{cases}
\mathbf{H}^{(t+1)} = \mathcal{M}_{\phi_\theta^{(t)}}(\mathbf{H}^{(t)}, A), \\[4pt]
\theta^{(t+1)} = \theta^{(t)} - \eta \nabla_\theta \mathcal{L}_{\mathrm{field}}(\phi_\theta^{(t)}, \mathbf{H}^{(t+1)}),
\end{cases}
\label{eq:app_dyn_sys}
\end{equation}
where $\mathcal{M}_{\phi}$ denotes the field-modulated message passing defined in Eq.~\eqref{eq:message_passing}.
We are interested in the convergence of this coupled iteration to a stationary point of the joint objective
$\mathcal{J}(\mathbf{H}, \theta) = \mathcal{L}_{\mathrm{task}}(\mathbf{H}) + \lambda \mathcal{L}_{\mathrm{field}}(\phi_\theta, \mathbf{H})$.

\paragraph{Message passing as a contraction.}
Consider the propagation operator $\mathcal{M}_\phi$ for a fixed field $\phi$.
With bounded, non-negative weights $w_{ij}\in(0,1]$ and typical choices of $f$ and $\psi$ (e.g., ReLU activations and layer-normalized linear transformations), $\mathcal{M}_\phi$ is non-expansive and often contractive in practice.
In particular, if $f$ and $\psi$ are $L_f$- and $L_\psi$-Lipschitz with $L_f L_\psi \max_i \sum_j w_{ij} < 1$, then $\mathcal{M}_\phi$ is a contraction, implying the existence of a unique fixed point $\mathbf{H}^*(\phi)$ that depends continuously on $\phi$.
Even when the contraction constant is only guaranteed to be $\le 1$ (non-expansive), the iteration tends to approach a well-defined limit under mild conditions.

\paragraph{Field optimization.}
Given the current representations $\mathbf{H}^{(t+1)}$, the field parameters $\theta$ are updated via gradient descent on $\mathcal{L}_{\mathrm{field}}$.
Under the smoothness of the field network $g_\theta$ and the coherence/separation regularizers, $\mathcal{L}_{\mathrm{field}}$ is $L_g$-smooth in $\theta$.
With a learning rate $\eta < 2/L_g$, the gradient step guarantees monotonic descent of the field objective, driving $\theta$ toward a stationary point.

\paragraph{Coupled system.}
Viewing the alternation as block coordinate descent on $\mathcal{J}(\mathbf{H},\theta)$, we have:
\begin{itemize}
    \item The $\mathbf{H}$-subproblem (field fixed) seeks a representation that minimizes the task loss under the current field-modulated graph;
    \item The $\theta$-subproblem (representation fixed) refines the field to better match the observed graph structure and the current representation.
\end{itemize}
When both subproblems are strictly convex in their respective variables (or at least satisfy the Kurdyka--\L{}ojasiewicz property), the overall scheme converges to a critical point of $\mathcal{J}$.
At this equilibrium, the field $\phi_{\theta^*}$ coherently explains the graph topology embedded in $\mathbf{H}^*$, and $\mathbf{H}^*$ incorporates the relational prior dictated by the field --- a symmetric partnership that embodies the field-graph co-evolution principle.

\paragraph{Remarks on practical convergence.}
In our experiments, the coupled dynamics stabilize within a few hundred epochs (see Fig.~\ref{fig:entropy_evolution}--\ref{fig:entropy_ratio}), corroborating the theoretical intuition.
The regularizers $\mathcal{L}_{\mathrm{coh}}$, $\mathcal{L}_{\mathrm{sep}}$, and $\mathcal{L}_{\mathrm{sym}}$ further promote well-conditioning by discouraging pathological field configurations.
Thus, while the formal contraction and convexity conditions may not be strictly met in every deep network, FGN's design consistently induces convergent behavior in practice.

%-------------------- B.4 Derivation of Practical Objective --------------------
\subsection{Derivation of the Practical Objective}
\label{app:derivation:objective}

Starting from the evidence lower bound (ELBO) for the marginal likelihood:
\begin{equation}
\log P(\mathcal{G} \mid X) \geq \mathbb{E}_{(\phi,\iota) \sim Q}[\log P(\mathcal{G} \mid \phi \circ \iota, X)] - \text{KL}(Q(\phi,\iota) \| P(\phi,\iota)),
\label{eq:app_elbo}
\end{equation}
we adopt a practical approximation strategy. Rather than maintaining full distributions over $\phi$ and $\iota$, we use point estimates: $Q(\phi,\iota) = \delta(\phi - \phi_\theta) \delta(\iota - \iota_\theta)$, where $\iota_\theta$ is implicitly encoded within the parameterization of $\phi_\theta$. This simplification reduces the ELBO to:
\begin{equation}
\log P(\mathcal{G} \mid \phi_\theta \circ \iota_\theta, X) - \text{KL}(\delta \| P).
\end{equation}

The KL divergence term is replaced by an explicit regularization term $\mathcal{R}(\phi_\theta)$ that encodes desired properties of the learned field. This leads to the regularized maximum likelihood objective:
\begin{equation}
\mathcal{L}(\theta) = -\frac{1}{|\mathcal{E}|}\sum_{(i,j)\in\mathcal{E}} \log P_{ij}(\phi_\theta) + \lambda \mathcal{R}(\phi_\theta),
\label{eq:app_practical_obj}
\end{equation}
where the first term corresponds to negative log-likelihood of observed edges, and the second term enforces field regularity.

The design of $\mathcal{R}(\phi_\theta)$ incorporates several inductive biases. First, we encourage field coherence by promoting similarity between connected nodes, implemented through $\mathcal{L}_{\text{coh}}$ in Eq.~\eqref{eq:coherence_loss}. Second, to prevent trivial solutions where all nodes converge to identical field values, we include a separation loss $\mathcal{L}_{\text{sep}}$ from Eq.~\eqref{eq:separation_loss} that encourages field variance. Third, to operationalize the co-evolutionary principle, we introduce a mutual information loss $\mathcal{L}_{\text{sym}}$ from Eq.~\eqref{eq:symmetric_loss}. These components are combined with task-specific loss $\mathcal{L}_{\text{task}}$ to form the complete FGN objective in Eq.~\eqref{eq:fgn_total_loss}.

%-------------------- B.5 Detailed ELBO Derivation --------------------
\subsection{Detailed Derivation of the Practical Objective}
\label{app:elbo_derivation}

Starting from the evidence lower bound (ELBO) for the marginal likelihood:
\begin{equation}
\log P(\mathcal{G} \mid X) \ge \mathbb{E}_{(\phi,\iota)\sim Q}[\log P(\mathcal{G} \mid \phi \circ \iota, X)] - \mathrm{KL}(Q(\phi,\iota) \| P(\phi,\iota)).
\label{eq:app_elbo_full}
\end{equation}

We adopt a mean-field factorization $Q(\phi,\iota)=Q_\phi(\phi)Q_\iota(\iota)$ and further restrict to point estimates:
\begin{equation}
Q_\phi(\phi) = \delta(\phi - \phi_\theta), \quad Q_\iota(\iota) = \delta(\iota - \iota_\theta),
\end{equation}
where $\phi_\theta$ is a neural network parameterized by $\theta$, and $\iota_\theta$ is implicitly defined through the same network.
Under this Dirac approximation, the expectation term simplifies to $\log P(\mathcal{G} \mid \phi_\theta \circ \iota_\theta, X)$, and the KL divergence becomes:
\begin{equation}
\mathrm{KL}(\delta(\phi-\phi_\theta)\delta(\iota-\iota_\theta) \| P(\phi,\iota)) = -\log P(\phi_\theta,\iota_\theta) + \text{const}.
\end{equation}

Assuming a prior that factorizes as $P(\phi,\iota)=P(\phi)P(\iota)$ and that $P(\iota)$ is uniform, we have $-\log P(\phi_\theta,\iota_\theta) = -\log P(\phi_\theta) + \text{const}$.
A common choice for $P(\phi)$ is a Gaussian process prior that enforces smoothness, which leads to a regularizer of the form $\mathcal{R}(\phi_\theta) = \frac{1}{2}\|\nabla \phi_\theta\|^2$ in the continuous limit.
In our discrete graph setting, we approximate this with graph-based smoothness penalties:
\begin{equation}
-\log P(\phi_\theta) \approx \lambda_1 \sum_{(i,j)\in\mathcal{E}} (\mathbf{E}_i - \mathbf{E}_j)^2 + \lambda_2 \cdot (-\mathrm{Var}(\mathbf{E})) + \lambda_3 \cdot (-\widehat{I}(\mathbf{H};\mathbf{E})),
\end{equation}
where $\mathbf{E}_i = \phi_\theta(x_i)$, and the three terms correspond to coherence, separation, and mutual information regularization, respectively.
Thus the objective in Eq.~\eqref{eq:practical_objective} emerges naturally, with $\lambda$ absorbing the prior scales.

%==================== C. Relation to Standard Graph Attention ====================
\section{Relation to Standard Graph Attention}
\label{app:attention_relation}

FGN's field-modulated message passing bears a superficial resemblance to standard graph attention mechanisms such as GAT~\cite{GAT}, where the influence of neighbor $j$ on node $i$ is computed as $\alpha_{ij} = \operatorname{softmax}_j(e_{ij})$ with $e_{ij}$ being a function of the two node features.
It is therefore important to clarify the conceptual and architectural distinctions that set the proposed method apart.

\begin{enumerate}
    \item \textbf{Global vs.\ local computation.}
    In GAT, attention coefficients $\alpha_{ij}$ are computed independently per layer and per node, based purely on the current-layer representations $\mathbf{h}_i^{(\ell)}$ and $\mathbf{h}_j^{(\ell)}$.
    There is no notion of a persistent, globally coherent quantity that links the scores across different neighborhoods or layers.
    In contrast, FGN's weights $w_{ij}$ are derived from a single scalar field $\mathbf{E} = \sigma(\mathrm{MLP}_{\theta_E}(X))$ that is shared across the entire graph and across all message-passing layers.
    This field is a standalone object with its own state and parameters $\theta_E$, rather than an ephemeral by-product of the node features at a particular layer.

    \item \textbf{Field evolution vs.\ static attention.}
    GAT's attention function is trained purely via the downstream task loss (e.g., cross-entropy), without any explicit objective that encourages structural coherence of the attention patterns themselves.
    In FGN, the field $\mathbf{E}$ is learned through a dedicated, information-theoretic objective $\mathcal{L}_{\mathrm{field}}$ that combines coherence, separation, and mutual information terms~(Eq.~\ref{eq:fgn_total_loss}).
    This objective drives $\mathbf{E}$ to become a geometrically consistent landscape over the graph, a property that is not required of standard attention scores.
    The ablation study in Table~\ref{tab:ablation_across_datasets} provides direct evidence: adding a static field (``+ Static field'') to a GCN backbone improves performance, but the largest gain is obtained when the field is allowed to co-evolve with the representations (``+ Dynamic co-evolution'').
    This confirms that the field's continuous evolution, not merely the use of node-wise gating, is responsible for the observed improvements.

    \item \textbf{Interpretability.}
    Because the field $\mathbf{E}$ is a single, globally defined scalar function, it can be visualized and queried to inspect the ``relational potential'' of any node, independent of the specific neighborhood or layer.
    This yields a degree of interpretability that is absent in standard multi-head attention, where the scores are high-dimensional, layer-specific, and difficult to aggregate into a unified picture.
    As shown in the main experiments, the field values naturally correlate with community structures and structural roles, providing a continuous substrate that explains the discrete connectivity.
\end{enumerate}

In summary, while the mathematical form of $w_{ij} = \exp(-\xi\|\mathbf{E}_i - \mathbf{E}_j\|^2)$ is reminiscent of an attention weight, FGN fundamentally departs from attention-based GNNs by promoting the field to a first-class citizen that participates in its own co-evolutionary loop.
The ablation results substantiate that this reification of the latent field, rather than the specific weighting formula, is the key factor behind the method's performance and robustness.

%==================== D. Experimental Datasets ====================
\section{Experimental Datasets Description}
\label{app:experimental_datasets_description}

This appendix provides detailed information about the benchmark datasets used in all experiments throughout this paper.

Table~\ref{tab:dataset_stats} presents the basic statistical characteristics of the nine datasets used in this study.

\begin{table}[htbp]
\centering
\caption{Statistics of the Experimental Datasets}
\label{tab:dataset_stats}
\setlength{\tabcolsep}{2.5pt}
\small
\begin{tabular}{ccccccc}
\hline
Dataset & Nodes & Edges & Features & Classes & Train & Test \\
\hline
Cora & 2,708 & 5,429 & 1,433 & 7 & 140 & 1,000 \\
CiteSeer & 3,327 & 4,732 & 3,703 & 6 & 120 & 1,000 \\
PubMed & 19,717 & 44,338 & 500 & 3 & 60 & 1,000 \\
Chameleon & 2,277 & 31,421 & 2,325 & 5 & 1,092 & 455 \\
ogb-arxiv & 169,343 & 1,166,243 & 128 & 40 & 90,941 & 48,603 \\
Computers & 13,752 & 245,861 & 767 & 10 & 200 & 1,000 \\
Photo & 7,650 & 119,081 & 745 & 8 & 160 & 1,000 \\
Texas & 183 & 309 & 1,703 & 5 & 87 & 150 \\
Cornell & 183 & 295 & 1,703 & 5 & 87 & 150 \\
\hline
\end{tabular}
\end{table}

\textbf{Cora Dataset}
The Cora dataset is a classic citation network consisting of 2,708 academic papers in the field of machine learning. These papers are divided into 7 categories, for example Case-Based Reasoning and Genetic Algorithms. Each paper is represented as a node, and the 5,429 directed edges between nodes indicate citation relationships. Each paper is described by a 1,433-dimensional binary word vector, indicating the presence or absence of corresponding words from a dictionary.

\textbf{CiteSeer Dataset}
The CiteSeer dataset is a citation network collected from the CiteSeer digital library, containing 3,327 computer science papers categorized into 6 classes. The dataset includes 4,732 citation edges. Node features are 3,703-dimensional binary word vectors, similarly indicating the inclusion of specific words.

\textbf{PubMed Dataset}
The PubMed dataset comprises 19,717 medical research papers in the field of diabetes, categorized into 3 classes. This dataset is relatively large-scale, containing 44,338 citation edges. Unlike the other two datasets, PubMed uses 500-dimensional TF-IDF feature vectors to represent each paper.

\textbf{Chameleon Dataset}
The Chameleon dataset is a webpage network collected from Wikipedia, containing 2,277 webpages categorized into 5 classes, for example Health and Politics. The dataset includes 31,421 hyperlink edges between webpages. Each webpage is represented by a 2,325-dimensional feature vector encoding informative nouns from the page content, capturing topical information for node classification.

\textbf{ogb-arxiv Dataset}
The ogb-arxiv dataset is a large-scale citation network from the Open Graph Benchmark, containing 169,343 computer science papers from arXiv categorized into 40 subject areas. The dataset includes 1,166,243 citation edges, making it a challenging large-scale benchmark. Each paper is represented by a 128-dimensional feature vector obtained by averaging the embeddings of words in the title and abstract.

\textbf{Computers Dataset}
The Computers dataset is a product co-purchase network from Amazon, containing 13,752 computers and computer accessories categorized into 10 product types. The dataset includes 245,861 edges representing frequent co-purchases between products. Each product is represented by a 767-dimensional feature vector derived from product reviews, capturing semantic information about the items.

\textbf{Photo Dataset}
The Photo dataset is another Amazon product co-purchase network, containing 7,650 photography-related products categorized into 8 types. The dataset includes 119,081 co-purchase edges. Similar to the Computers dataset, each product is represented by a 745-dimensional feature vector based on review information, enabling product category classification.

\textbf{Texas Dataset}
The Texas dataset is a small-scale webpage network from the WebKB collection, containing 183 webpages from the University of Texas website categorized into 5 classes. The dataset includes 309 hyperlink edges between webpages. Each webpage is represented by a 1,703-dimensional feature vector encoding the presence of specific words, with this dataset known for its heterophilic structure where connected pages often belong to different categories.

\textbf{Cornell Dataset}
The Cornell dataset is another WebKB webpage network, containing 183 webpages from the website of Cornell University categorized into 5 classes. The dataset includes 295 hyperlink edges. Similar to the Texas dataset, it uses 1,703-dimensional binary word vectors as features and exhibits heterophilic characteristics, making it a challenging benchmark for graph neural networks.

%==================== E. Algorithm Pseudocode ====================
\section{Algorithm Pseudocode}
\label{app:algorithm}

Algorithm \ref{alg:fgn} presents the complete training procedure of the FGN.

\begin{algorithm}[htbp]
\caption{Field-informed Graph Network Training}
\label{alg:fgn}
\begin{algorithmic}[1]
\REQUIRE Graph $\mathcal{G} = (\mathcal{V}, \mathcal{E})$ with adjacency $A$ and node features $X$, Number of layers $L$
\ENSURE Learned node representations $\mathbf{H}$

    \STATE Initialize parameters
    \STATE Compute energy field $\mathbf{E}$ with Eq.~\eqref{eq:energy_field}
    \STATE $\mathbf{H}^{(0)} \gets X$
    \FOR{$\ell = 1$ to $L$}
        \FOR{each node $i \in \mathcal{V}$}
            \STATE Compute weights $w_{ij}$ for $j \in \mathcal{N}(i)$ with Eq.~\eqref{eq:simple_attention}
            \STATE Aggregate messages and update $\mathbf{h}_i^{(\ell)}$ with Eq.~\eqref{eq:message_passing}
        \ENDFOR
    \ENDFOR
    \STATE $\mathbf{H} \gets \mathbf{H}^{(L)}$
    \STATE Compute predictions $\hat{Y} \gets \text{Classifier}(\mathbf{H})$
    \STATE Compute total loss $\mathcal{L}$ using Eq.~\eqref{eq:fgn_total_loss}
    \STATE Compute gradients
    \STATE Update parameters
    \STATE RETURN $\mathbf{H}$
\end{algorithmic}
\end{algorithm}

%==================== F. Experimental Configuration ====================
\section{Experimental Configuration}
\label{app:experiment_details}

\subsection{Computational Environment}
\label{app:computational_environment}

All experiments were conducted on a dedicated computing server equipped with an Intel Xeon Gold 6348 processor operating at 2.60GHz, 251GB of DDR4 RAM, and dual NVIDIA GeForce RTX 4090 GPUs providing a total of 48GB VRAM. The software environment utilized Python 3.8.20 with PyTorch 2.4.1 and CUDA 12.1.

\subsection{Hyperparameter Settings}
\label{app:hyperparameter_settings}

The hyperparameter configurations are summarized in Table \ref{tab:hyperparameter_settings}.

\begin{table}[htbp]
\centering
\caption{Hyperparameter settings for the experimental datasets.}
\label{tab:hyperparameter_settings}
\setlength{\tabcolsep}{2.5pt}
\small
\begin{tabular}{lccccc}
\hline
Dataset & Learning rate & Weight decay  & Epochs & Dropout & Hidden dim  \\
\hline
Cora & 0.01 & $5 \times 10^{-4}$ & 200 & 0.5 & 16 \\
CiteSeer & 0.01 & $5 \times 10^{-4}$ & 200 & 0.5 & 16 \\
PubMed & 0.005 & $5 \times 10^{-5}$ & 500 & 0.5 & 32\\
Chameleon & 0.01 & $5 \times 10^{-5}$ & 400 & 0.5 & 32 \\
ogb-arxiv & 0.01 & $3 \times 10^{-4}$ & 500 & 0.5 & 128 \\
Computers & 0.005 & $5 \times 10^{-5}$ & 500 & 0.5 & 32 \\
Photo & 0.005 & $5 \times 10^{-5}$ & 500 & 0.5 & 32 \\
Texas & 0.01 & $3 \times 10^{-5}$ & 200 & 0.5 & 16 \\
Cornell & 0.01 & $5 \times 10^{-5}$ & 200 & 0.5 & 16 \\
\hline
\end{tabular}
\end{table}

In our experiments, we found FGN to be relatively stable to the choice of regularization weights. The default values with $\alpha=0.5$, $\beta=0.1$, $\gamma=0.05$ worked well across all datasets. Performance variations remained within $\pm$0.5\% when $\alpha$ and $\beta$ varied in the range $[0.1, 1.0]$ and $\gamma$ in $[0.01, 0.1]$ on Cora.

\subsection{Complexity Analysis}
FGN achieves computational complexity $O(L \cdot (|\mathcal{E}|d^2 + NFd_h))$, where $L$ is the number of message passing layers, $|\mathcal{E}|$ is the number of edges, $d$ is the hidden dimension, $N$ is the number of nodes, $F$ is the input feature dimension, and $d_h$ is the field network width. The term $O(|\mathcal{E}|d^2)$ corresponds to the field-modulated message passing operations per layer, and $O(NFd_h)$ corresponds to computing the scalar energy field $\mathbf{E}$ via the MLP in Eq.~\eqref{eq:energy_field}.

Since $L$ is typically a small constant (2--3 layers) and real-world graphs are sparse with $|\mathcal{E}| \sim O(N)$, FGN maintains linear scalability. Moreover, $d_h \ll d$ makes the overhead from the field computation negligible.

%==================== G. Ablation Study Details ====================
\section{Ablation Study Details}
\label{app:ablation_details}

We define each cumulative variant in Table~\ref{tab:ablation_across_datasets} as follows:

\begin{itemize}
    \item \textbf{GCN (baseline)}: Standard two-layer GCN without any field modulation.
    \item \textbf{+ Static field}: A scalar energy field $\mathbf{E} = \sigma(\text{MLP}(X))$ is computed from node features and used to compute weights $w_{ij} = \exp(-\xi \|\mathbf{E}_i - \mathbf{E}_j\|^2)$ as in Eq.~\eqref{eq:simple_attention}, with \textbf{frozen} field parameters.
    \item \textbf{+ Field-modulated diffusion}: The field $\mathbf{E}$ is jointly trained via $\mathcal{L}_{\text{coh}}$ (Eq.~\eqref{eq:coherence_loss}), but without bidirectional co-evolution.
    \item \textbf{+ Dynamic co-evolution}: Full iterative coupling enabled via Eq.~\eqref{eq:field_update}.
    \item \textbf{+ Entropy balance}: Separation loss $\mathcal{L}_{\text{sep}}$ (Eq.~\eqref{eq:separation_loss}) added.
    \item \textbf{+ Hierarchical constraints}: Mutual information loss $\mathcal{L}_{\text{sym}}$ (Eq.~\eqref{eq:symmetric_loss}) added.
\end{itemize}

%==================== H. Robustness Experiment Details ====================
\section{Robustness Experiment Details}
\label{app:robustness_details}

All robustness experiments are conducted on the Cora dataset. Results are averaged over 10 independent random seeds.

\paragraph{Feature dropping.}
For a given drop ratio $p \in \{0, 0.2, 0.4, 0.6, 0.8, 0.9\}$, we randomly mask a fraction $p$ of the feature dimensions to zero independently for each node.

\paragraph{Feature attack (Gaussian noise).}
We add i.i.d.\ Gaussian noise: $\tilde{x}_{i,j} = x_{i,j} + \epsilon_{i,j}$, where $\epsilon_{i,j} \sim \mathcal{N}(0, \sigma^2 \cdot \operatorname{Var}(X_{\cdot,j}))$, with $\sigma \in \{0.1, 0.2, 0.5, 1.0\}$.

\paragraph{Edge removal.}
We randomly delete a fraction $q \in \{0, 0.2, 0.4, 0.6\}$ of existing edges. No connectivity restoration is applied.

\paragraph{Edge addition.}
We randomly add a fraction $r \in \{0, 0.2, 0.4, 0.6\}$ of non-edges from the complement graph.

%==================== I. Visualization Details ====================
\section{Visualization Details}
\label{app:visualization}

This appendix provides the t-SNE visualizations of node embeddings on the Cora dataset for FGN and several baseline methods. All visualizations use the same hyperparameters for fair comparison. Fig.~\ref{fig:visualization} shows that FGN forms tight, well-separated clusters aligned with class labels, while baseline methods exhibit more inter-class mixing. This qualitative evidence supports the effectiveness of our proposed field-graph co-evolution mechanism.

\begin{figure}[ht]
  \centering
  \begin{subfigure}{0.2\linewidth}
    \includegraphics[width=\linewidth]{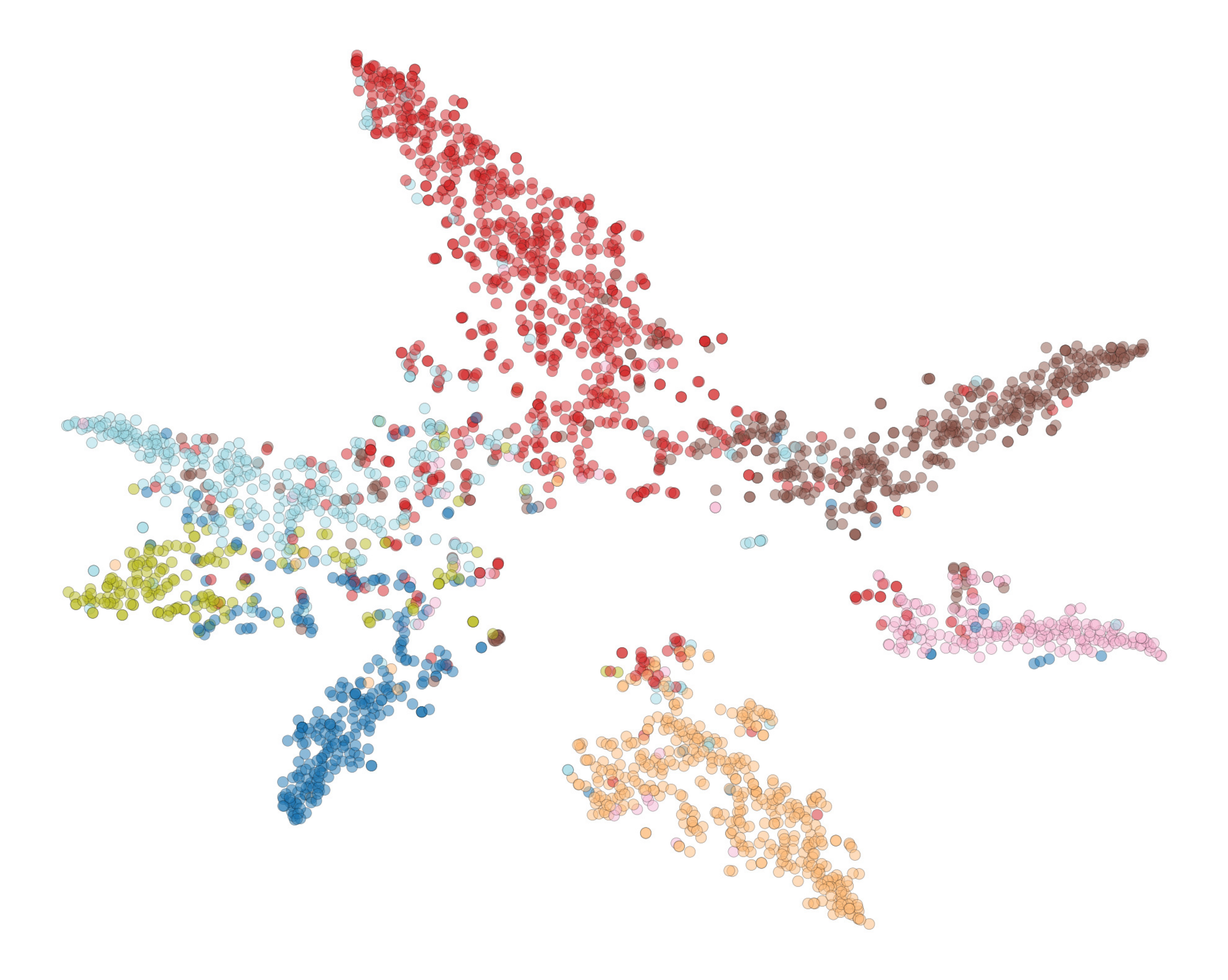}
    \caption{FGN}
    \label{subfig:fga}
  \end{subfigure}
  \begin{subfigure}{0.2\linewidth}
    \includegraphics[width=\linewidth]{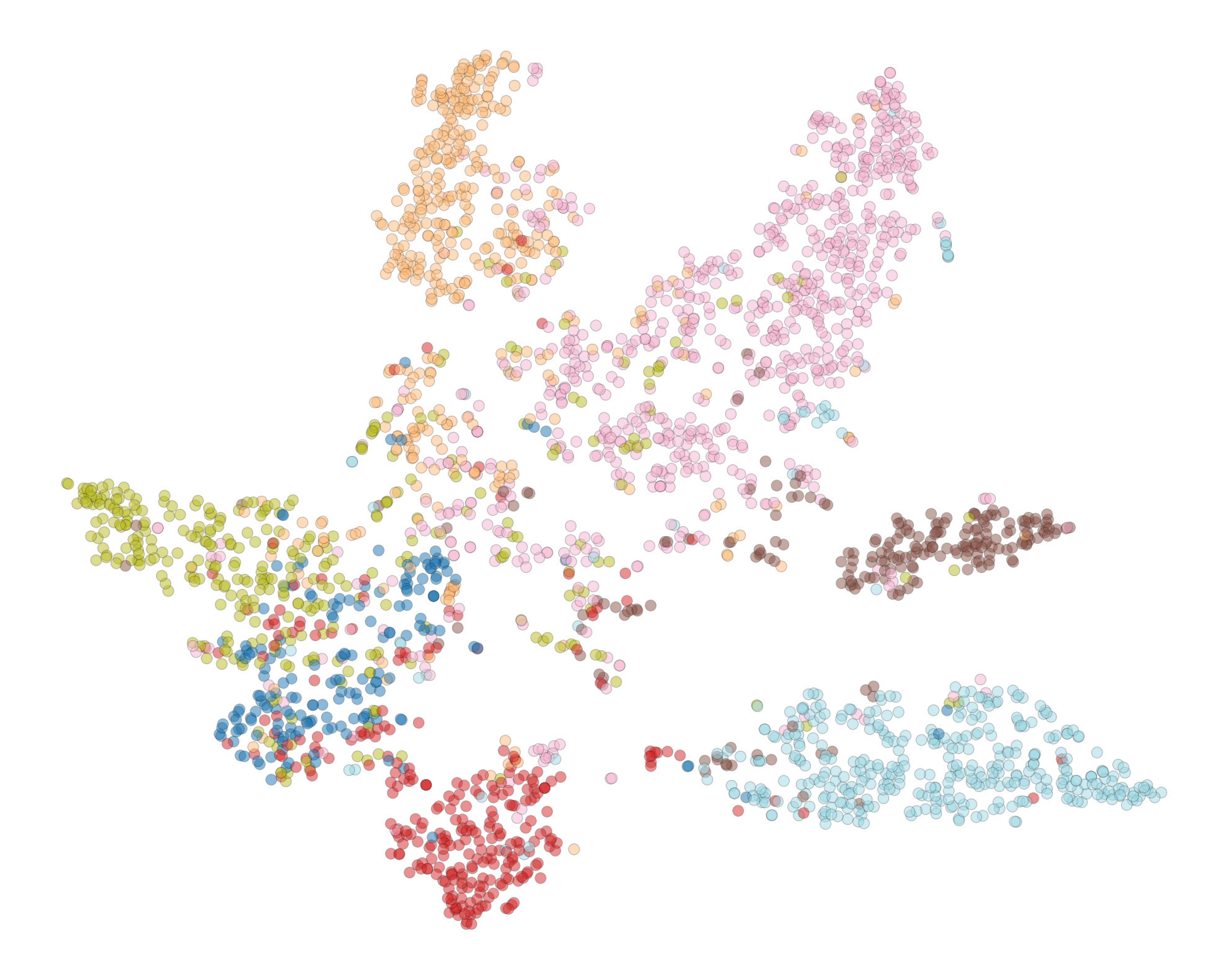}
    \caption{GCN}
    \label{subfig:gcn}
  \end{subfigure}
  \begin{subfigure}{0.2\linewidth}
    \includegraphics[width=\linewidth]{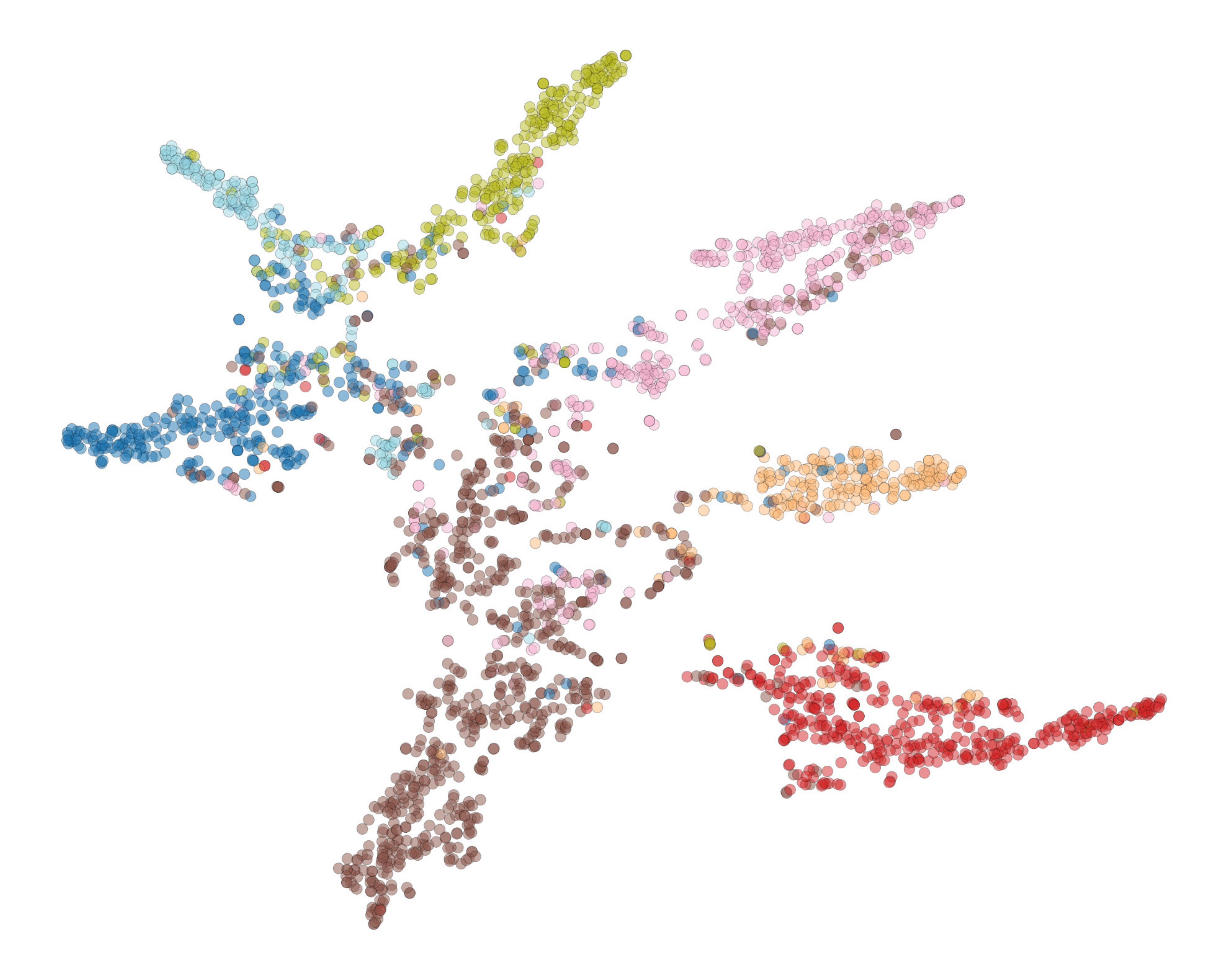}
    \caption{GAT}
    \label{subfig:gat}
  \end{subfigure}
  \begin{subfigure}{0.2\linewidth}
    \includegraphics[width=\linewidth]{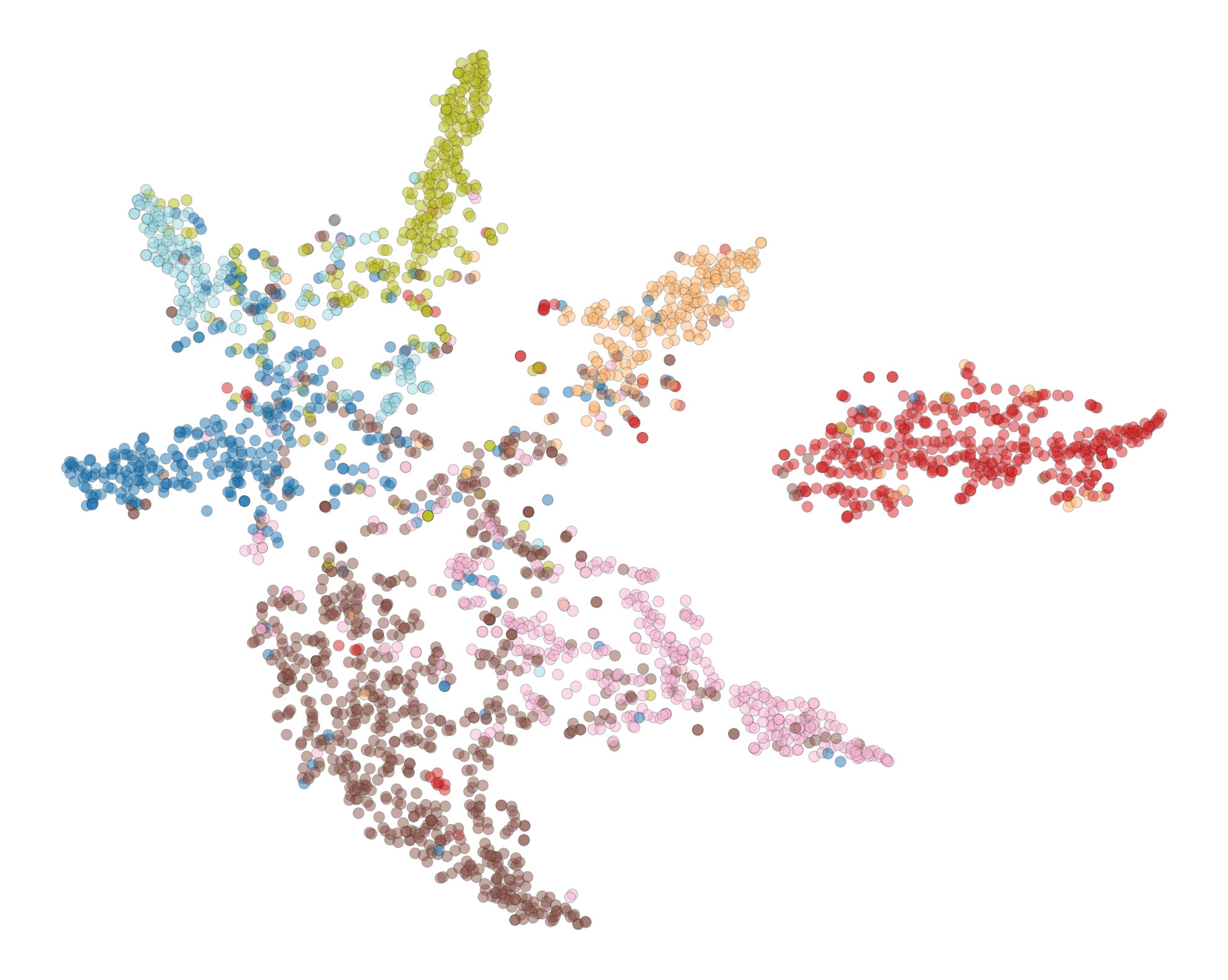}
    \caption{PGBSF}
    \label{subfig:pgbsf}
  \end{subfigure}
  \begin{subfigure}{0.2\linewidth}
    \includegraphics[width=\linewidth]{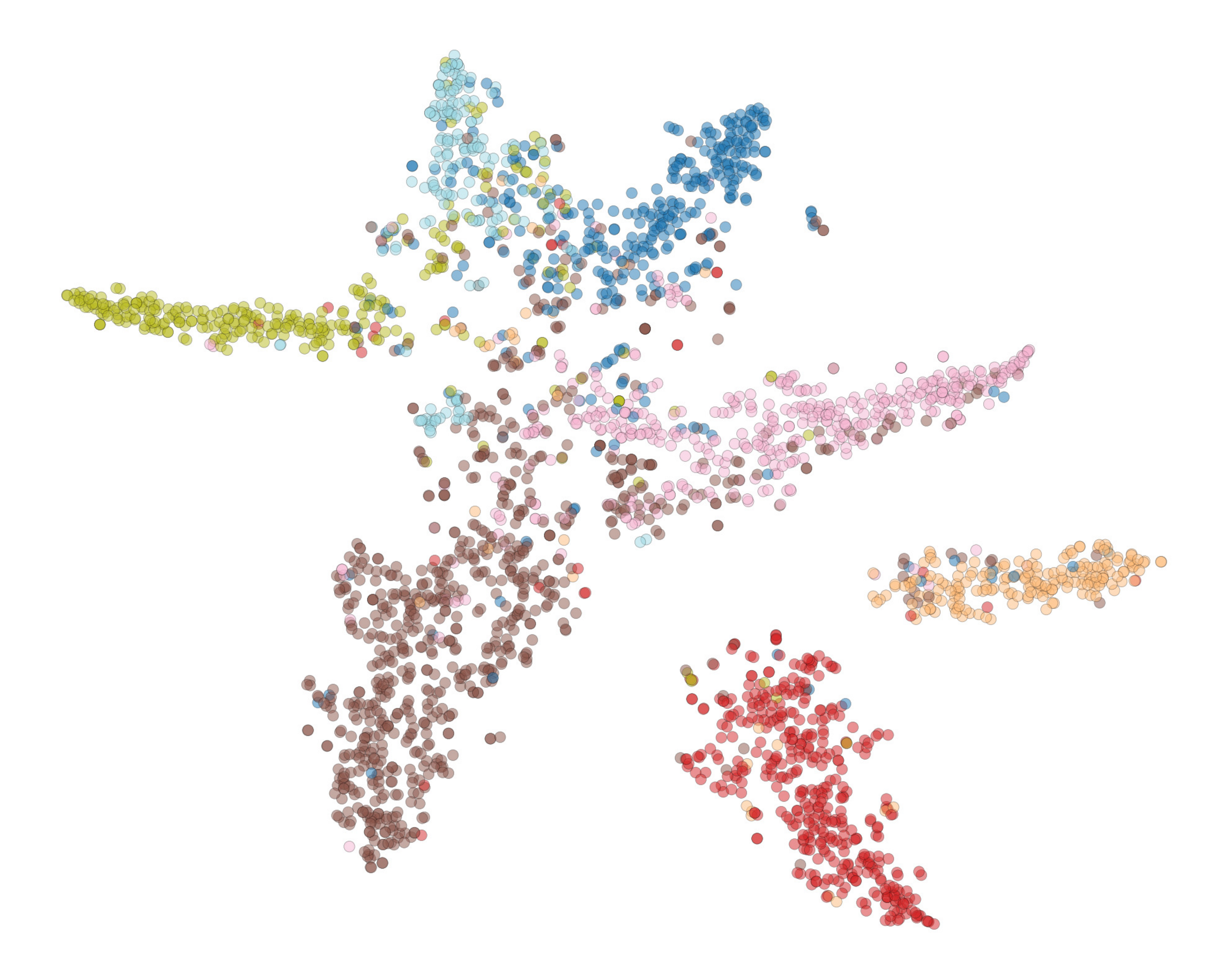}
    \caption{ELU-GCN}
    \label{subfig:elu-gcn}
  \end{subfigure}
  \begin{subfigure}{0.2\linewidth}
    \includegraphics[width=\linewidth]{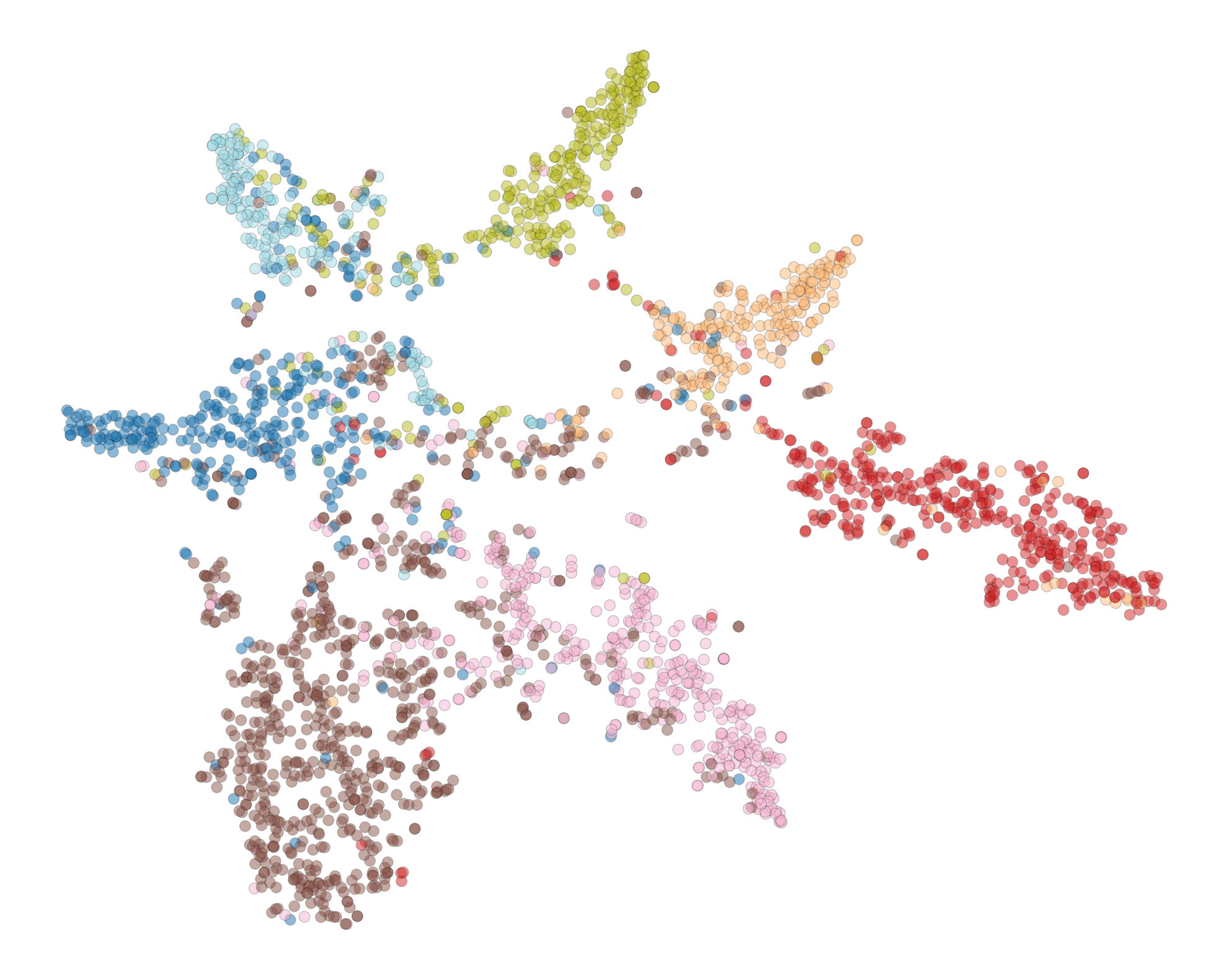}
    \caption{FTCP}
    \label{subfig:ftcp}
  \end{subfigure}
  \begin{subfigure}{0.2\linewidth}
    \includegraphics[width=\linewidth]{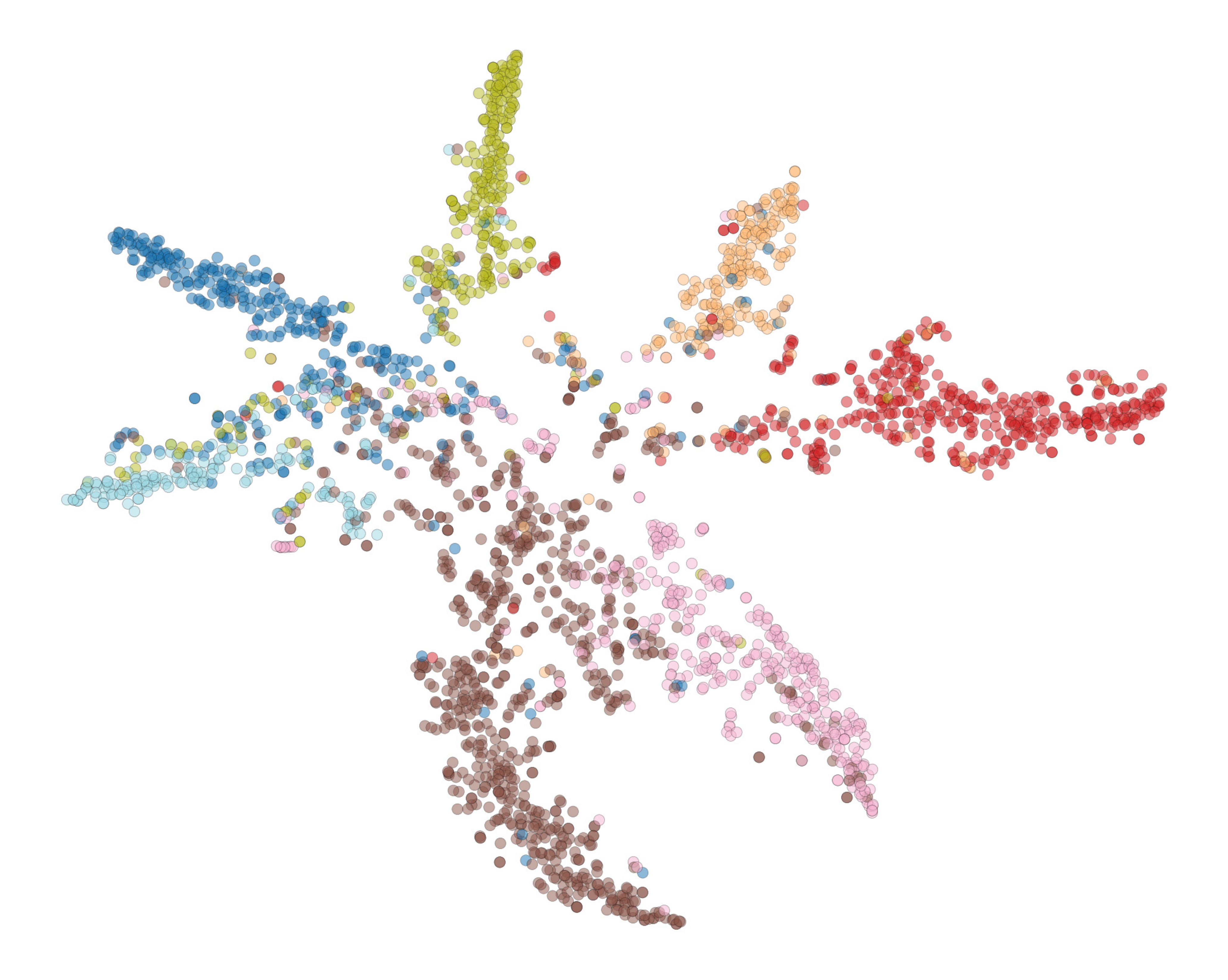}
    \caption{MOGCN}
    \label{subfig:mogcn}
  \end{subfigure}
  \begin{subfigure}{0.2\linewidth}
    \includegraphics[width=\linewidth]{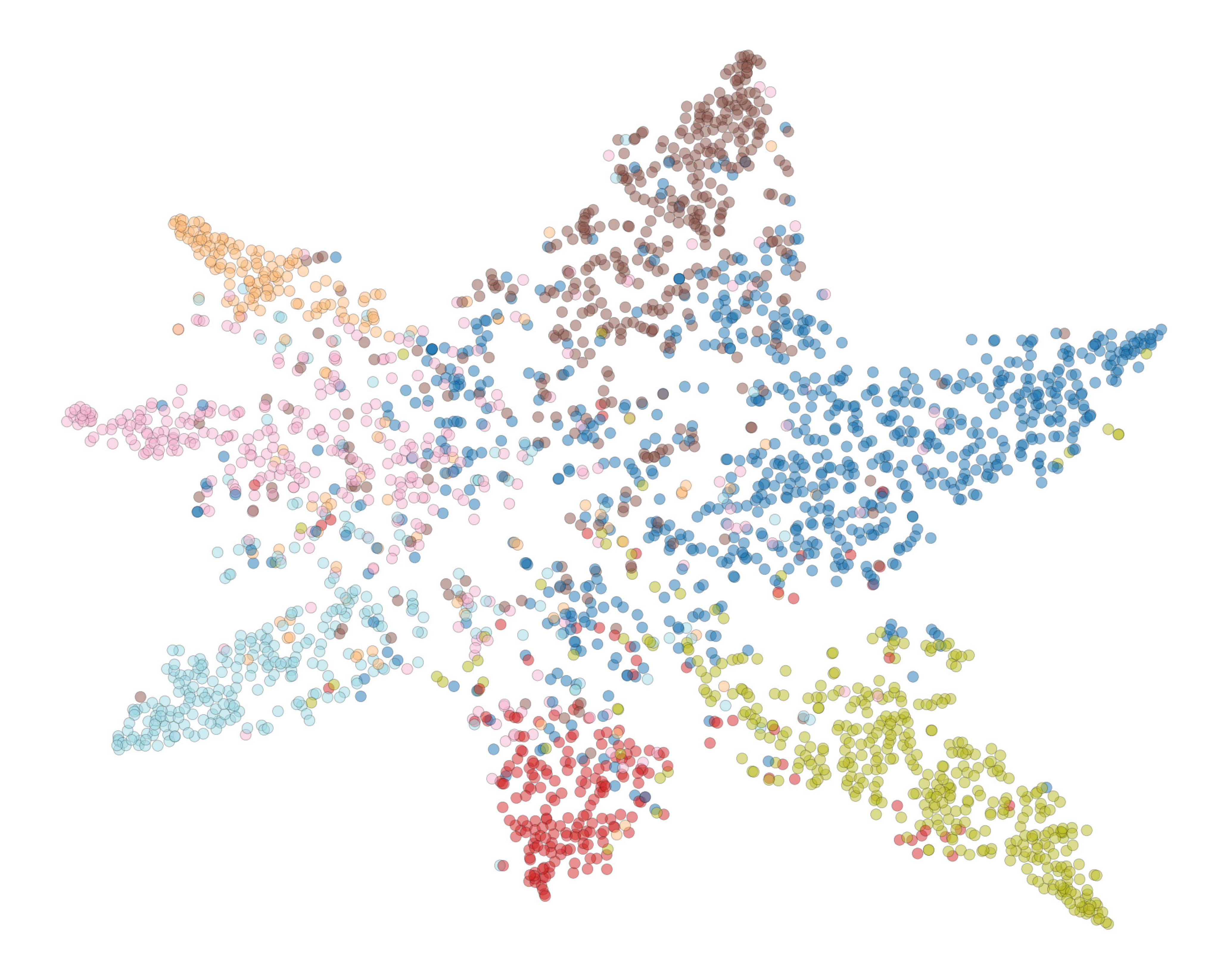}
    \caption{HGRN}
    \label{subfig:hgrn}
  \end{subfigure}
  \caption{t-SNE visualization comparison on Cora dataset.}
  \label{fig:visualization}
\end{figure}

\end{document}